\documentclass{article}
\usepackage[final]{neurips_2025}
\usepackage{amsmath}
\usepackage{amssymb,amsfonts,amsthm}
\usepackage{mathrsfs}
\usepackage{mathtools}
\usepackage{algorithm}
\usepackage{algorithmic}
\usepackage{xcolor}
\usepackage{pgfplots}
\pgfplotsset{compat=1.18}
\usepackage{tikz}
\usetikzlibrary{arrows.meta, positioning,patterns}
\usepackage{dsfont}
\usepackage{thmtools}
\usepackage{thm-restate}
\usepackage{enumitem}
\usepackage{dsfont}
\usepackage{graphicx}
\usepackage{subcaption}
\usepackage{sidecap}
\usepackage{accents}
\usepackage{wrapfig}

\newtheorem{definition}{Definition}[]
\newtheorem{theorem}{Theorem}[section]
\newtheorem{observation}{Observation}

\newtheorem{corollary}{Corollary}[]
\newtheorem{proposition}{Proposition}[]
\newtheorem{example}{Construction}[]
\newtheorem{problem}{Problem}[]

\usepackage{hyperref}
\usepackage{xcolor}
\hypersetup{
  colorlinks   = true, %Colours links instead of ugly boxes
  urlcolor     = blue, %Colour for external hyperlinks
  linkcolor    = blue, %Colour of internal links
  citecolor   = blue %Colour of citations
}
\usepackage[capitalize, nameinlink]{cleveref}

\usepackage[utf8]{inputenc} % allow utf-8 input
\usepackage[T1]{fontenc}    % use 8-bit T1 fonts
\usepackage{hyperref}       % hyperlinks
\usepackage{url}            % simple URL typesetting
\usepackage{booktabs}       % professional-quality tables
\usepackage{amsfonts}       % blackboard math symbols
\usepackage{nicefrac}       % compact symbols for 1/2, etc.
\usepackage{microtype}      % microtypography
\usepackage{xcolor}

\newcommand{\adapter}{\mathsf{A}}

\newcommand{\loss}{\mathcal{L}}
\newcommand{\lossinfonce}{\mathcal{L}^{\mathsf{InfoNCE}}}
\newcommand{\lossinfonceu}{\mathcal{L}^{\mathsf{U-InfoNCE}}}
\newcommand{\losssiglip}{\mathcal{L}^{\mathsf{Sig}}}
\newcommand{\losssigliprb}{\mathcal{L}^{\mathsf{RB-Sig}}}
\newcommand{\losstriplet}{\mathcal{L}^{\mathsf{Triplet}}}

\newcommand{\margin}{\mathsf{m}}
\newcommand{\relativebias}{\mathsf{b}_{\mathsf{rel}}}
\newcommand{\conv}{\mathsf{conv}}
\newcommand{\cone}{\mathsf{cone}}
\newcommand{\dualcone}{\mathsf{dualcone}}

\DeclareMathOperator*{\expect}{{\rm I}\kern-0.18em{\rm E}}

\DeclareMathOperator*{\prob}{{\rm I}\kern-0.18em{\rm P}}

\newcommand{\sizesphericalcode}{\mathsf{N}_{\mathsf{SC}}}
\newcommand{\logsizesphericalcode}{\mathsf{E}_{\mathsf{SC}}}
\newcommand{\sizeconstellation}{\mathsf{N}_{\mathsf{MRB}}}
\newcommand{\logsizeconstellation}{\mathsf{E}_{\mathsf{MRB}}}

\title{Global Minimizers of Sigmoid Contrastive Loss}
\author{%
  Kiril Bangachev\thanks{Author names appear in alphabetical order by last name.}\thanks{Supported by NAE Grand Challenge Vest Fellowship.} \\
  \texttt{kirilb@mit.edu} \\
  \And 
  Guy Bresler\thanks{Supported by NSF award 2428619.}\\
  \texttt{guy@mit.edu}\\
  \And
  Iliyas Noman\\
  \texttt{iliyas@mit.edu}\\
  \And
  Yury Polyanskiy\thanks{Supported by United States Air
Force Research Laboratory and the United States Air Force
Artificial Intelligence Accelerator Cooperative Agreement Number FA8750-19-2-1000.}\\
  \texttt{yp@mit.edu}\\ \\
  Department of Electrical Engineering and Computer Science\\
  Massachusetts Institute of Technology\\
  Cambridge, MA, 02139 \\}

\date{May 2025}

\begin{document}

\maketitle

\begin{abstract}
The meta-task of obtaining and aligning representations through contrastive pretraining is steadily gaining importance since its introduction in CLIP and ALIGN. 
% Common representations to be synchronized include image and text data, student and teacher models, or multiple augmentations of the same image. 
In this paper we theoretically explain the advantages of synchronizing with \emph{trainable inverse temperature and bias} under the sigmoid loss, as implemented in the recent SigLIP and SigLIP2 models of Google DeepMind. Temperature and bias can drive the loss function to zero 
for a rich class of configurations that we call $(\margin, \relativebias)$-Constellations. $(\margin, \relativebias)$-Constellations are a novel combinatorial object related to spherical codes and are parametrized by a margin $\margin$ and relative bias $\relativebias$. We use our characterization of constellations to theoretically justify the success of 
SigLIP on retrieval,
to explain
the modality gap present in SigLIP and CLIP,
and to identify the necessary dimension for producing high-quality representations.  
Finally, we propose a reparameterization of the sigmoid loss with explicit relative bias, which improves training dynamics in experiments with synthetic data. All code is available at 
\href{https://github.com/BangachevKiril/RepresentationLearningTheory.git}{RepresentationLearningTheory/SigLIP}.
% This parameterization yields several practical recommendations, including in the contexts of \textit{synchronizing with a locked encoder} and
%\textit{synchronizing more than two modalities.}
\end{abstract}

% \begin{abstract}
% Increasingly common in machine learning is the meta-task of synchronizing representations -- of image and text data, student and teacher models, multiple augmentations of the same image. We theoretically explain the advantages of synchronizing with \emph{trainable temperature and bias} under the sigmoid loss, as implemented in the recent SigLIP models of Google DeepMind. Temperature and bias can drive the loss function to zero 
% for a rich class of configurations that we call $(\mathsf{m}, \mathsf{rb})$-Constellations. $(\mathsf{m}, \mathsf{rb})$-Constellations are a novel combinatorial object related to spherical codes and are parametrized by a margin $\mathsf{m}$ and relative bias $\mathsf{rb}.$ We use the characterization of constellations to theoretically justify the success of 
% SigLIP on retrieval,
% to explain
% the modality gap present in SigLIP,
% and to identify the necessary dimension for producing high-quality representations.  
% We also propose a re-parameterization of the sigmoid loss with explicit relative bias. This parameterization yields several practical recommendations, including in the contexts of \textit{synchronizing with a locked encoder} and
% \textit{synchronizing more than two modalities.}
% \end{abstract}
\section{Introduction}
% Aligning; image-text, modalities, teacher-student distillation/self-distillation; Variety of tasks.

% Poor theoretical. For InfoCNE loss, .. For SigLIP, computationally superior no batch-norr, only d<<N

% Unqiue minima, trainabl hyper-params. We characterize the zero-loss representations with trainable params. The solution geometry 
\paragraph{Background.} \emph{Synchronizing representations} is an increasingly important meta-task in modern machine learning, appearing in several qualitatively different contexts. Models that operate jointly on visual and language data necessitate a synchronization of the representations of images and text \cite{radford21clip,desai2021redcaps,srinivasan21wit,chanpinyo21conceptual12M, chao21align,li22blip,schuhmann2022laionb,kakaobrain2022coyo700m,zhai22lit,hu22imagecaptioning,wang22git,chen2023pali,zhai23siglip,tschannen2025siglip2} and sometimes of additional modalities as well such as audio, thermal data, and others \cite{girdhar23imagebind}. State-of-the-art vision models based on self-distillation rely on aligning the representations of augmentations of the same image \cite{chen20contrastivelearningvisual,he20momentumcontrastive,caron21dino}. Likewise, aligning the representations of data produced by teacher and student networks \cite{Tian2020Contrastivedistillation,Giakoumoglou2024Distillation} has been proposed as a method for distillation. Similarly to the teacher-student setup, the field of backward-compatible learning aims to synchronize the features produced by new models with features of already trained old models \cite{vivek22backwardcompatible,biondi23backwardcompatible,Shen2020backardcompatible,jaeckle2023bacwardcompatible}.

To mathematically formalize the task of synchronizing two representations, suppose that there are $N$ data pairs $\{(X_i,Y_i)\}_{i = 1}^N\in (\mathcal{X}\times \mathcal{Y})^{\otimes N}.$ For concreteness, one can think of $\mathcal{X}$ as the space of images, $\mathcal{Y}$ as the space of text, and $(X_i, Y_i)$ satisfy a \emph{correspondence relation} such as the fact that they are a true image-caption pair.
The goal is to train neural network \emph{encoders} $f_\theta:\mathcal{X}\longrightarrow\mathbb{R}^{d}$ and $g_\phi:\mathcal{Y}\longrightarrow\mathbb{R}^{d}$ in such a way that 
the embeddings produced by them capture the correspondence relation. Such synchronization is usually achieved via minimizing a certain \emph{contrastive loss}.

Despite the prevalence of the task of synchronizing representations, there is still limited understanding of \emph{what loss function to use, how to choose its hyper-parameters}, and \emph{what properties of the synchronized embeddings are desirable}. Theoretical results focus mostly on two loss functions -- the InfoNCE loss \cite{isola20understanding,chuang2020debiased,robinson2021contrastive,weinan22simplexsymmetry,LU2022neuralcollapse,papyan20neuralcollapse,gupta2024structuring} and the Sigmoid Loss \cite{lee24analysissiglip,LU2022neuralcollapse}, which both depend on temperature and bias hyper-parameters. 
While these prior works have yielded useful insights, they leave important gaps in our understanding of representation synchronization:
% While largely influential, prior works leave important gaps in our understanding of representation synchronization.

\emph{1. Currently understood regimes for the number of represented objects $N$ compared to the dimension of representations $d$ do not reflect practice}. To the best of our knowledge, in all prior theoretical works, either $d\ge N,$
    % which is rather unrealistic as one key idea behind learning representations is to have \emph{succinct} descriptions of objects. Or, 
    or $N$ approaches $+\infty$ for a fixed value of $d.$ 
    % Again, this regime is rather unsatisfactory as it implies that representation of certain different objects are necessarily similar to each other. 
    As a comparison, the SigLIP2 model embeds text and images in $d\approx 10^3$ dimensions and operates with
    a dataset of size $N \approx 10^{10}$ \cite{tschannen2025siglip2}.
    Thus the practically relevant regime -- in which $d\ll N\ll  2^{d}$ --
    % is the \emph{polynomial regime} $N = d^{\alpha}$ where $\alpha>1$ which 
    is not captured by prior work. The different regimes exhibit crucially different behaviors: practically relevant phenomena such as the \emph{modality gap} \cite{liang2022mindthemodalitygap} only arise when $N> d$, as we show in \cref{thm:sperablemodalities}.
    
\emph{2. The optimal configurations identified by prior works are too rigid}. For example, works in the regime $N\le d$ typically suggest a simplex structure of the embeddings of each modality \cite{weinan22simplexsymmetry,LU2022neuralcollapse,lee24analysissiglip}. This does not explain {what the minimizing configurations} are \emph{when one modality is pretrained and locked}. In the regime $N\longrightarrow +\infty,$ existing results typically suggest a perfect alignment between different representations. Again, this may be too stringent since it has been proposed that ``different modalities may contain different information'' \cite{huh24platonic}. In fact, empirical work suggests that even after synchronization, representations of text and images are completely disjoint, a phenomenon known as \emph{the modality gap} \cite{liang2022mindthemodalitygap,fahim2025its}. 

\paragraph{Our Contributions.}
% \nbyp{I think the key difference is that we have fixed rb, trainable t and b was before.}\kiril{In terms of theory, there is nothing for trainable t and b. }
In the current work, we address these gaps by analyzing the sigmoid loss \emph{with trainable inverse temperature and bias parameters}, as used in Google's SigLIP models \cite{zhai23siglip,tschannen2025siglip2} and Gemma 3 \cite{gemmateam2025gemma3technicalreport}. Making bias and temperature trainable is a key 
departure from prior theoretical work \cite{isola20understanding,LU2022neuralcollapse,weinan22simplexsymmetry,lee24analysissiglip} and leads to novel theoretical guarantees and practical recommendations. To the best of our knowledge, the only prior work considering trainable temperature is \cite{gui2025multi}, but via the very different angle of the intrinsic dimension of the data. We first introduce the sigmoid loss and then describe our contributions.

The sigmoid loss for $U_i = f_\theta(X_i)$ and $V_i = g_\phi(Y_i)$ and inverse temperature\footnote{Previous works on synchronizing with sigmoid loss such as \cite{zhai23siglip,tschannen2025siglip2,lee24analysissiglip} call $t$ the \emph{temperature}. We call it \emph{inverse temperature} to be more consistent with statistical physics terminology.} $t$ and bias $b$ is:
\begin{equation}
    \label{eq:defsiglip}
    \begin{split}
    & \losssiglip(\theta, \phi;t,b) = 
    \sum_{i = 1}^N \log \Big(1 + \exp(- t \langle U_i, V_i\rangle + b)\Big) + 
    \sum_{i \neq j} \log \Big(1 + \exp(t \langle U_i, V_j\rangle - b)\Big)\,.
    \end{split}
\end{equation}
The first part of the loss encourages the embedding of an image and its caption to be similar, while the second part encourages mismatched image-caption pairs to be dissimilar. 

\vspace*{-.3cm}
\paragraph{1. The Geometry of Zero-Loss Configurations.} Our work is the first to rigorously characterize global minima in representation synchronization tasks in the practical regime $ N\gg d$. 

We show that the SigLIP loss—with trainable temperature and bias—can be driven to zero by a rich family of solutions, which we fully characterize in terms of two novel geometric quantities -- the \emph{margin}  $\margin\ge 0$ and the \emph{relative bias} $\relativebias.$  Formally, a \emph{$(d, \margin, \relativebias)$-Constellation}\footnote{We usually omit the parameter $d$ since it is clear from the context and only write ``$(\margin, \relativebias)$-constellation.''} $\{(U_i, V_i)\}_{i = 1}^N\in \mathbb{S}^{d-1}$ is defined by the following inequalities:
\begin{equation}
\label{eq:mrconstellation}
    \begin{split}
        & \langle U_i, V_i\rangle\ge \margin + \relativebias\quad\quad\forall i,\\
        & \langle U_i, V_j\rangle\le -\margin + \relativebias\quad\quad\forall i\neq j.\\
    \end{split}
\end{equation}

\begin{wrapfigure}{r}{0.5\textwidth}
\vspace*{-0.7cm}
\includegraphics[width=.5\linewidth]{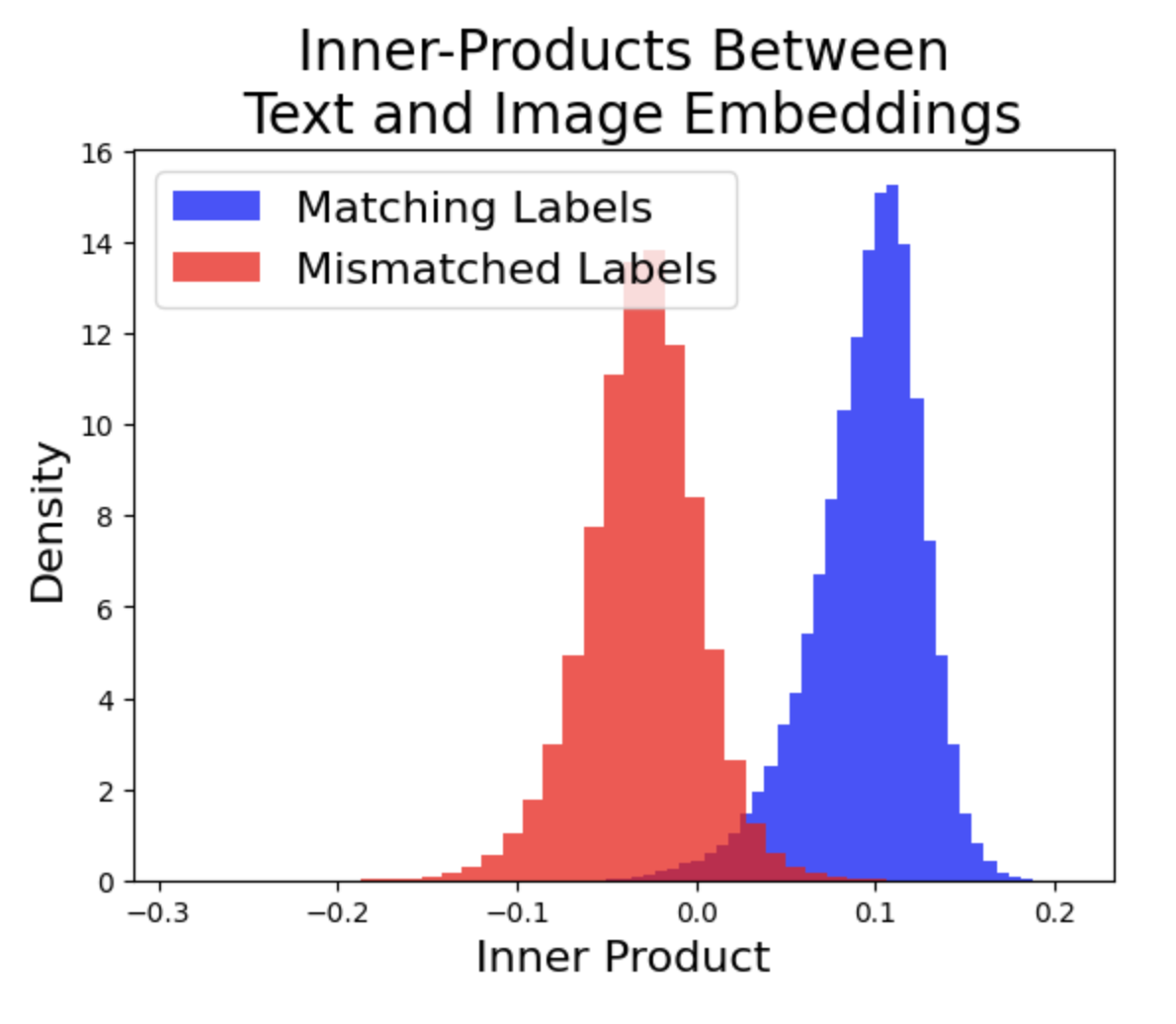}
    \vspace*{-0.45cm}
    \caption{\small {Distribution of inner products between image and text embeddings from the ImageNet validation set using the $B/16$ $224\times 224$ SigLIP model available at \href{https://huggingface.co/google/siglip-base-patch16-224}{HuggingFace}.}}
    \label{fig:siglip_ipseparation}
     \vspace*{-0.4cm}
\end{wrapfigure}

The existence of such $\margin, \relativebias,$ which one can observe is equivalent to \emph{the inner product separation}
$\min_i\langle U_i, V_i\rangle \ge \max_{i\neq j}\langle U_i, V_j\rangle,$ is a necessary and sufficient condition for $\{(U_i, V_i)\}_{i = 1}^N$ to be a global minima of the sigmoid loss with trainable inverse temperature and bias. We show that this is nearly satisfied in practice for the SigLIP model trained on real images and text -- See Figure \ref{fig:siglip_ipseparation}\footnote{The experimental details for all plots as well as further experiments are in \cref{sec:experiments}.}. Surprisingly, any configuration satisfying this condition is also a global minimum for the \emph{triplet loss}, see Observation \ref{obs:tripletlossproof}. We interpret the margin and relative bias in \cref{sec:geometryofzeroloss}.

Condition \cref{eq:mrconstellation} is a more quantitative version, including margin, of the ``globally thresholdabe'' property, analyzed in \cite{weller2025theoretical}, that $\langle U_i, V_i\rangle >\relativebias$ for positive pairs and $\langle U_i, V_j\rangle <\relativebias$ for negative pairs. Thus, our analysis states that configurations that are globally thresholdable are exactly the ones that are minimizers of the sigmoid loss with trainable temperature and bias. Carrying out a similar analysis for the InfoNCE loss, on the other hand, we discover  a ``row-wise thresholdable'' geometry instead \cite{weller2025theoretical}. Namely, global minimizers of (one-sided) InfoNCE with trainable inverse temperature are exactly characterized by:\footnote{Notice that the characterization is rather different from ``uniformity'' +``alignment'' from \cite{isola20understanding}. This is because \cite{isola20understanding} is for fixed inverse temperature and in the asymptotic limit $N\longrightarrow + \infty.$ Many practical models like OpenCLIP, however, have a trainable inverse temperature as in our set-up.}
\begin{equation}
\label{eq:InfoNCEGlobMinima}
    \begin{split}
        & \langle U_i, V_i\rangle \ge \relativebias(i) + m\qquad \forall i,\\
        & \langle U_i, V_j\rangle \le \relativebias(i) - m\qquad \forall i\neq j.\\
    \end{split}
\end{equation}

\begin{figure}[!htb]
    \centering
    \includegraphics[width=\linewidth]{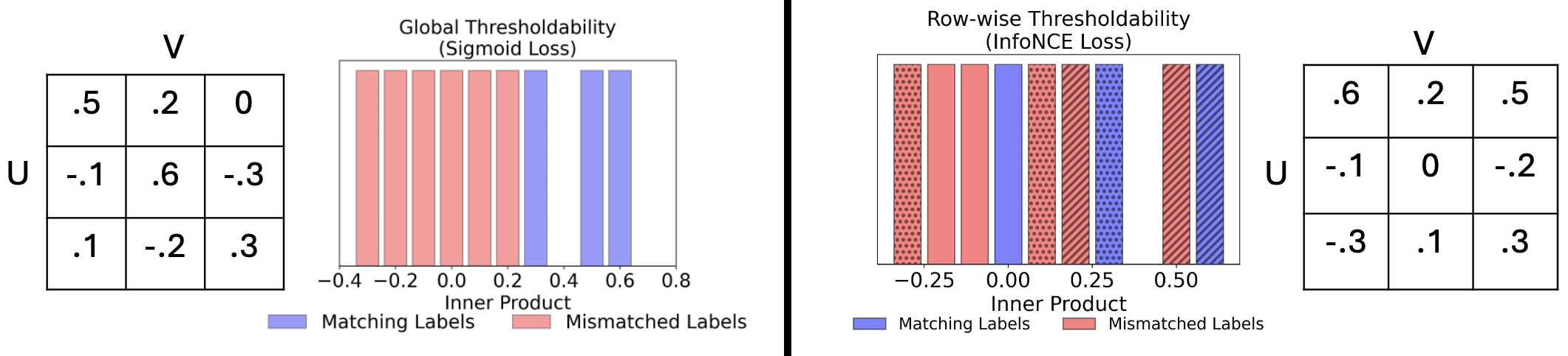}
    \caption{Examples of zero-loss configurations for Sigmoid loss (left) and InfoNCE (right), highlighting the difference in geometries.}
    \label{fig:infoncevssiglip}
\end{figure}

In practice, one needs to choose a dimension for the encoders which has large enough ``capacity'' to hold the embeddings of great many pairs $U_i,V_i$. However, despite intuitive notion that capacity should increase with dimension, to the best of our knowledge no such quantitative characterization was available before our work. Formally, we define the following combinatorial problem and make partial progress in \cref{sec:cardinalitybounds} via a connection to spherical codes. 

\begin{problem}
\label{problem:cardinalitybounds}
For a given $\margin\ge0, \relativebias\in [-1,1],$
find the largest number of points $N =\sizeconstellation(d, \margin,\relativebias)$ such that there exist $2N$ vectors $\{(U_i,V_i)\}_{i = 1}^N\in \mathbb{S}^{d-1}$ satisfying \eqref{eq:mrconstellation}.
\end{problem}

\pagebreak

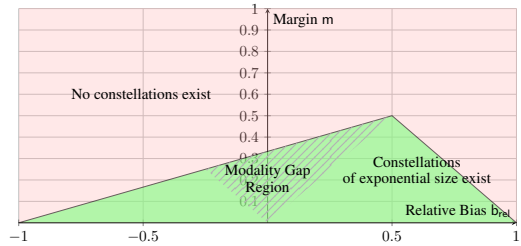
\begin{wrapfigure}{r}{.5\textwidth}
\vspace*{-.5cm}
\resizebox{\linewidth}{!}{
\begin{tikzpicture}
  \begin{axis}[
    xlabel={Relative Bias $\relativebias$},
    ylabel={Margin $\margin$},
    xlabel style={at={(axis description cs:0.5,-0.1)}, anchor=north},
    ylabel style={at={(axis description cs:-0.1,0.5)}, anchor=south},
    xmin=-1, xmax=1,
    ymin=0, ymax=1,
    axis lines=middle,
    enlargelimits=false,
    width=13cm,
    height=6.5cm,
    xtick={-1, -0.5, 0, 0.5, 1},
    ytick={0, 0.1, 0.2, 0.3, 0.4, 0.5,.6, .7, .8,.9,1},
    grid=both,
    clip=true
  ]

  % Red background with transparency
  \addplot [
    fill=red!30,
    draw=none,
    opacity=0.3
  ] coordinates {
    (-1,0)
    (1,0)
    (1,1)
    (-1,1)
    (-1,0)
  };

  % Green triangle with transparency
  \addplot [
    fill=green!50,
    draw=black,
    opacity=0.6
  ] coordinates {
    (-1,0)
    (0.5,0.5)
    (1,0)
    (-1,0)
  };

  % Modality Gap Region with pattern
  \addplot [
    pattern=north east lines,
    pattern color=black!30,
    draw=none
  ] coordinates {
    (-1/4,1/4)
    (0,0)
    (0.5,0.5)
    (-1/4,1/4)
  };

  % Labels
  \node[align=center] at (axis cs:0.6,0.24) { {Constellations}\\of exponential size exist
  %\\ Theorem \ref{thm:lowerbondsviasphericalcodes}
  };
  \node[align=center] at (axis cs:-0.5,0.6) { No constellations exist
  %\\ \cref{thm:upperboundsmargin}
  };
  \node[align=center] at (axis cs:-0,0.2) { {Modality} Gap\\Region
  %\cref{thm:sperablemodalities}
  };
  \end{axis}
\end{tikzpicture}}
%\vspace*{-.8cm}
\caption{\small{Region of possible $(\margin,\relativebias)$-Constellations. In red is the impossible region, in which no large configurations are possible (\cref{thm:upperboundsmargin}). In green is the region where constellations of exponential size exist (Theorem~\ref{thm:lowerbondsviasphericalcodes} and Theorem ~\ref{thm:upperboundconst}). In the shaded region we prove that a modality gap exists (Theorem~\ref{thm:sperablemodalities}).}}
\label{fig:possiblemarginbiasregion}
\vspace*{-.9cm}
\end{wrapfigure}

\paragraph{2. Success of Zero-Loss Configurations on Downstream Tasks.} In \cref{cor:perfectretrieval}, we use the characterization of zero-loss configurations to show that a standard nearest neighbor search on \emph{any $(\margin,\relativebias)$-Constellation gives perfect retrieval}, {even though typically there is no perfect alignment between the two representations.} Increasing the margin $\margin$ of a constellation makes retrieval robust to larger approximation errors.
This is important in practice since retrieval is often performed via an \emph{approximate nearest neighbor search} for computational efficiency \cite{xiong2021approximate,khattab20colbert,macdonald21ann}.

\paragraph{3. The Modality Gap: Synchronize, do not Align.} The analysis of \cite{isola20understanding} suggests alignment between representations when training via the InfoNCE loss -- the representations of the word ``cat'' and the image of a cat should (nearly) coincide. Yet, it has been empirically observed that there is a \emph{modality gap} \cite{liang2022mindthemodalitygap,fahim2025its}. The representations of images and text -- synchronized via the InfoNCE loss in CLIP -- do not align, but rather belong to fully disjoint, linearly separable regions. Furthermore, this is not caused by the difference between architecture of image and text encoders, as initially thought, but rather directly by virtue of (approximately) minimizing InfoNCE loss~\cite{fahim2025its}.

\begin{wrapfigure}{r}{0.27\textwidth}
\vspace*{-0.5cm}
\includegraphics[width = \linewidth]{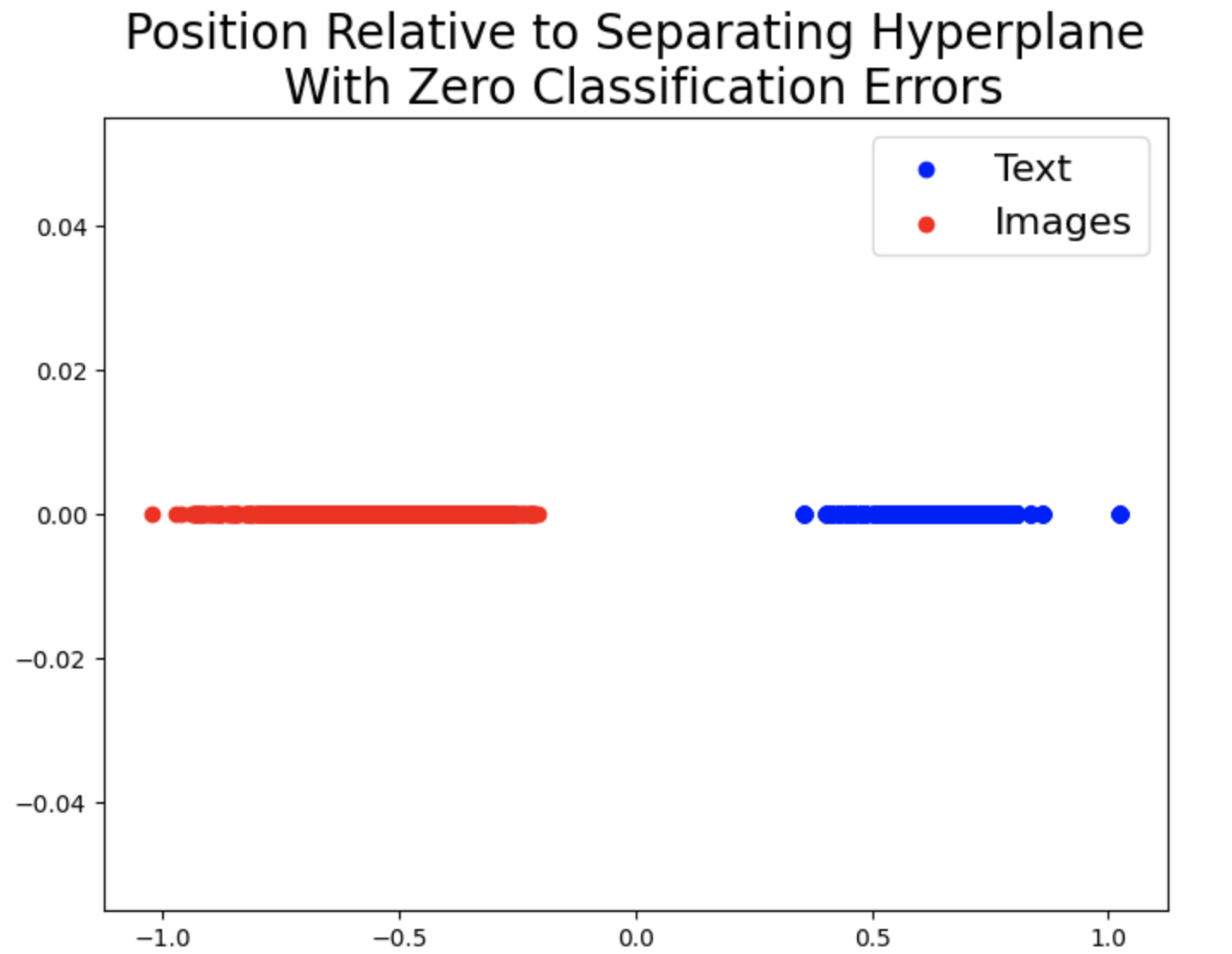}
\caption{\small{Modality gap in SigLIP on ImageNet data with the B/16 model with $224\times 224$ resolution. We find a perfect linear separator using the perceptron algorithm.}}    
\label{fig:modalitygasiglip}
\vspace*{-0.5cm}
\end{wrapfigure}

We shed light on this empirical discovery and prove in \cref{thm:sperablemodalities} that linear separability between modalities holds for any zero-loss configuration of the sigmoid loss in the practically relevant regime $N>d$ when $|\relativebias|<\margin$ (Figure \ref{fig:possiblemarginbiasregion}) and for the InfoNCE loss when $|\relativebias(i)|<\margin \forall i$ (see \cref{eq:InfoNCEGlobMinima}).  
We verify our findings by performing experiments with 8 different SigLIP models from Hugging Face  on the ImageNet dataset (models given in \cref{tab:siglip_results_5percent}). We observe perfect linear separability of image and text embeddings for all models.
From a philosophical point of view, as ``different modalities may contain different information'' \cite{huh24platonic}, it is only natural that they be represented in disjoint parts of the space.

We leverage the modality gap to build a linear adapter %based on the modality gap 
which can be used towards synchronizing representations when one encoder is locked. This is the reason why we use the name \emph{representation synchronization} rather than representation alignment: alignment between modalities is neither achieved nor necessarily desired. In cases where the modality gap hurts performance on downstream tasks, the concurrent work \cite{lee2025generalized} suggests a simple fix.

\paragraph{4. Implications of The Solution Geometry in Practice: Relative Bias Parameterization of Sigmoid Loss.} We propose a parametrization of the sigmoid loss that depends on the \emph{relative bias} rather than the bias in Definition~\ref{def:relativebiasparam}. The relative bias parametrization has the following advantages:

\emph{1. Locked Representation:} For example, in LiT \cite{zhai22lit}, the image encoder is already trained and locked and we want to synchronize the text encoder with it. The sigmoid loss with trainable parameters in the relative bias parametrization allows us to find a zero-loss configuration for text and images \emph{regardless} of the image encoder. In Observation \ref{obs:lockedapadapterimplicit}, we show that trainable relative bias and inverse temperature provide a mechanism to implicitly add linear adapters on top of the two encoders as in Figure \ref{fig:lockedencoder}. The linear adapters can alternatively be used for synchronizing with a locked modality. 
    % \nbyp{... "in which case any other popular contrastive losses can be used then."}\kiril{Not sure about this.. If we plug into InfoNCE loss I don't know what will happen..} 
   The adapters we propose in Observations~\ref{obs:lockedapadapterimplicit} and \ref{obs:multipleapadapterimplicit} extend the \emph{Double-Constant Embedding Model} of \cite{lee24analysissiglip}.

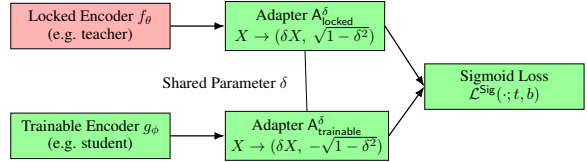
\begin{wrapfigure}{r}{0.55\textwidth}
\vspace*{-0.7cm}
\resizebox{\linewidth}{!}{
\begin{tikzpicture}[
  encoder/.style={draw, minimum width=3.5cm, minimum height=1cm, align=center},
  adaptor/.style={draw, fill=green!40, minimum width=3.5cm, minimum height=1cm, align=center},
  loss/.style={draw, fill=green!40, minimum width=3.5cm, minimum height=1cm, align=center},
  arrow/.style={-{Latex}, thick},
  ]
% Top branch
\node[encoder, fill=red!30] (locked) at (0,1.8) {Locked Encoder $f_\theta$\\(e.g. teacher)};
\node[adaptor, right=1.2cm of locked] (adapt_top) {Adapter $\adapter^\delta_{\mathsf{locked}}$\\ \(X \rightarrow (\delta X,\ \sqrt{1 - \delta^2})\)};

% Bottom branch
\node[encoder, fill=green!40, below=1.4cm of locked] (trainable) {Trainable Encoder $g_\phi$\\(e.g. student)};
\node[adaptor, right=1.2cm of trainable] (adapt_bot) {Adapter $\adapter^\delta_{\mathsf{trainable}}$\\ \(X \rightarrow (\delta X,\ -\sqrt{1 - \delta^2})\)};

% Loss box: horizontally right of adapt_top, vertically centered between locked and trainable
\node[right=2.5cm of adapt_top] (virtual) { };
\node[loss, below=.7cm of virtual] (lossbox) {Sigmoid Loss\\
$\losssiglip(\cdot;t,b)$};

\node[below = 1cm of adapt_top.west] (shardparam) {Shared  Parameter $\delta$};

% Arrows
\draw[arrow] (locked) -- (adapt_top);
\draw[arrow] (trainable) -- (adapt_bot);
\draw[arrow] (adapt_top.east) -- (lossbox.west);
\draw[arrow] (adapt_bot.east) -- (lossbox.west);
\draw (adapt_top.south) -- (adapt_bot.north);
\end{tikzpicture}}
\vspace*{-0.6cm}
\caption{\small Implicit adapter in relative bias parameterization of sigmoid loss with a locked representation. The parameters $\phi, \delta, t,b$ in green blocks are trainable. Parameter $\theta$ is locked.}
\label{fig:lockedencoder}
\vspace*{-1.3cm}
\end{wrapfigure}
    
\emph{2. More than Two Modalities:} The framework of training with the relative bias parameterization also leads to theoretical guarantees for the global minima of synchronizing more than two modalities via the sigmoid loss. Again, in Observation \ref{obs:multipleapadapterimplicit} we show that 
    the parameterization implicitly captures the addition of a modality-dependent linear adapter to each encoder.  

\emph{3. Guiding Relative Bias:} Relative bias and margin, which are related by inequalities that we fully characterize in \cref{thm:lowerbondsviasphericalcodes,thm:upperboundsmargin}, control important properties of the synchronized representations such as 
    \emph{retrieval robustness} and \emph{the presence of a modality gap}. We observe empirically that in the usual sigmoid loss parameterization, Adam \cite{KingmaB14adam} finds configurations with a zero relative bias, thus limiting the set of trained representations. By adding a relative bias parameter and locking it, we can provably guide the zero loss configuration to a more diverse set of solutions.
    See \cref{appendix:biasparametrization}.

\section{Background and Prior Work}
% General intro to representation learning
\paragraph{Representation Learning And Synchronization.}
% Aligning: 1) modalities, 2) teacher-student, 3) dino-style
A key insight in \cite{radford21clip,chao21align} is that training a model to \emph{simultaneously} operate on multiple modalities (such as image and text) enables SOTA performance on individual modalities as well -- ``if you want to train the best vision model, you should
train not just on $N$ images but also on $M$ sentences'' \cite{huh24platonic}.
Several empirical approaches towards synchronizing multiple representations have been proposed, including CLIP \cite{radford21clip}, BLIP \cite{li22blip}, ALIGN \cite{chao21align}, LiT \cite{zhai22lit}, and SigLIP \cite{zhai23siglip,tschannen2025siglip2}. 
The task of synchronizing representations goes beyond synchronizing across different modalities such as image and text, but also includes synchronizing the representations of a student model to a teacher model with the purpose of distillation \cite{Tian2020Contrastivedistillation,Giakoumoglou2024Distillation} and self-distillation \cite{caron21dino}, and aligning the representations of data augmentations
\cite{ZhangRichard2016CIC,noroozi16jigsaw,gidaris2018unsupervisedrepresentationlearningpredicting,chen20contrastivelearningvisual,he20momentumcontrastive,caron21dino}. 

\paragraph{Formalizing Representation Synchronization.} In this paper, we consider unit-norm encoders $f_\theta: \mathcal{X}\longrightarrow \mathbb{S}^{d-1}, g_\phi: \mathcal{Y}\longrightarrow \mathbb{S}^{d-1}$ (unit norm representations are predominant in practice, see e.g. \cite{Schroff2015tripletloss,Parkhi15deepfacerec,liu2017sphereface,wang17normface,chen20contrastive,he20momentumcontrastive,tian20contrastivemultiview,zhai23siglip,tschannen2025siglip2} and others). Synchronizing the representations produced by $f_\theta, g_\phi$ is usually achieved via minimizing a certain loss function $\loss$ over the kernel produced by the embeddings:
\begin{equation}
\label{eq:generallosssynch}
\loss(\theta, \phi; \Gamma) = 
\loss(\{\langle f_\theta(x_i), g_\phi(y_j)\rangle\}_{1\le i, j \le N}; \Gamma)
\end{equation}
where $\Gamma$ is a set of hyper-parameters, typically involving (inverse) temperature and bias. The optimization is performed via batch first-order optimization methods such as SGD or Adam \cite{KingmaB14adam}. 

Depending on the targeted representations, one may choose the parameters over which the optimization is performed. If both $f_\theta, g_\phi$ are untrained, one may perform gradient descent on both $\phi, \theta$ in \eqref{eq:generallosssynch} as in CLIP \cite{radford21clip}. On the other hand, if one of the models -- say $f_\theta$ -- is already trained and trusted (for example, because it is the teacher model that we are trying to distill \cite{caron21dino,Giakoumoglou2024Distillation}), we do not update its parameters or only update a small adapter on top of it as in \cite{zhang22adapter,liu23adapter,GaoPeng2024Clipadapter,yang2024adapter,lu2024uniadapter,ebrahimi2024cromecrossmodaladaptersefficient}. Likewise, this is the case when one of the modalities has already been trained and is locked \cite{zhai22lit,rosenfeld2022ape,li22blip,li2023blip2}.

Besides choosing which parameters to update, one also needs to choose a concrete loss function $\loss.$ The choices depend on two factors: 1) What is the geometry of the desired minimizing configurations of representations?
2) How efficient is the computation of the loss in terms of the batch size? 

We focus on two different loss functions -- InfoNCE and sigmoid -- and now survey previous works on them. In order to understand the solution geometry of the minimizers, a typical assumption is that the underlying networks $f_\theta, g_\phi$ are sufficiently expressive and can encode any embedding $\{(U_i,V_i)\}_{i = 1}^N = \{(f_\theta(X_i),g_\phi(Y_i))\}_{i = 1}^N$ \cite{isola20understanding,weinan22simplexsymmetry,LU2022neuralcollapse,lee24analysissiglip}. We also adopt this approach.

\paragraph{Solution Geometry with InfoNCE.} The InfoNCE loss \cite{oord2019representationlearningcontrastivepredictive} with inverse temperature $t>0$ and bias $b$ is a special case of \eqref{eq:generallosssynch} and takes the following form 
\begin{equation}
    \label{eq:definfonce}
    \begin{split}
    & \lossinfonce(\{(U_i,V_i)\}_{i = 1}^N;t) = 
    - \frac{1}{N}\sum_{i = 1}^N \log \frac{\exp(t \langle U_i, V_i\rangle )}{\sum_{j}\exp(t \langle U_i, V_j\rangle)}
    - \frac{1}{N}\sum_{i = 1}^N \log \frac{\exp(t \langle U_i, V_i\rangle)}{\sum_{j}\exp(t \langle U_j, V_i\rangle)}.
    \end{split}
\end{equation}
It effectively takes a soft-max over all the rows and columns of the matrix $t U^TV,$ an interpretation that yields a connection to the InfoMax principle \cite{hjelm2018learning} as well as to maximizing point-wise mutual information \cite{tian20contrastivemultiview} and approximate sufficient statistics \cite{oko2025statisticaltheorycontrastivepretraining,lin2025statisticaltheorycontrastivelearning}.

The solution geometry has been characterized in the case $d\ge N+1.$ The global minimum loss is achieved when $U_i = V_i$ for each $i$ and $U_1, U_2, \ldots, U_n$ form a simplex \cite{weinan22simplexsymmetry,LU2022neuralcollapse}. When $N\longrightarrow +\infty,$ the minimizing measures converge to perfectly aligned ($U_i = V_i$ for all $i$) and uniform (the discrete measure corresponding to $\{U_i\}_{i = 1}^n$ converges weakly to the uniform measure). Several works including \cite{arora19classification,elst24contrastivepacbayes} take a different direction and analyze the global minima of the InfoNCE loss (and its symmetrization SimCLR) in terms of performance on downstream (linear) classification tasks instead. While such a geometric characterization is appealing from a practical point of view, these works also do not address the aforementioned gaps in our understanding. The results of \cite{arora19classification} hold in the regime $N\longrightarrow +\infty$ and \cite{elst24contrastivepacbayes} points out that
``temperature scaling in the SimCLR loss remains challenging.''

In a very different direction, recently it was also rigorously shown that the InfoNCE yields an optimal dimensionality reduction with input data from a Gaussian Mixture Model \cite{bansal2025understandingselfsupervisedlearninggaussian}.

\paragraph{Solution Geometry with Sigmoid Loss.} An alternative loss function used towards alignment is the sigmoid loss \cite{zhai23siglip,tschannen2025siglip2} defined in \eqref{eq:defsiglip}. 
One advantage of the sigmoid loss over InfoNCE is that it does not have a batch normalization term such as $\sum_{j \neq i}\exp(t \langle U_i, V_j\rangle - b)$ and, thus, every pair $(U_i, V_j)$ can be processed separately. This allows for parallel computation.
% This allows for a much greater flexibility in parallelizing computation. 

The solution geometry of configurations achieving global minimum loss has been characterized when $d\ge N$ \cite{lee24analysissiglip}. For a simplex $\{W_i\}_{i = 1}^N$ in $\mathbb{S}^{d-2}$ and some $\delta\in [0,1],$ it holds that $U_i = (\delta W_i, \sqrt{1-\delta^2}), V_i = (\delta W_i, - \sqrt{1-\delta^2})$ for each $i$, where the value of $\delta$ depends on the relationship between $t$ and $b.$ In most cases, either the representations collapse to perfectly aligned ($\delta = 1$ and $U_i = V_i$ for all $i$) or antipodal
($\delta = 0$ and $U_i = - V_i = (1,0,0\ldots,0)$ for all $i$). %We remark that 
The construction of 
\cite{lee24analysissiglip} is the basis for several of our results, including the adapters proposed in Observations~\ref{obs:lockedapadapterimplicit}, \ref{obs:multipleapadapterimplicit}.

\paragraph{Other loss functions.} Loss functions such as the triplet loss
\cite{Schroff2015tripletloss} and $f$-MICl \cite{lu2023fmicl}
have also been considered. We further discuss the triplet loss in \cref{sec:relationtotripletloss}.

% \nbyp{Definition of Sigmoid loss should appear earlier (before main results). We also need to say that actual training also updates t and b (how? every batch? every epoch?). This is crucial to understanding all of the "our results" section. This subsection should simply ref LCS24 and explain their result  then. }

\section{Main Results}
\subsection{Geometric Characterization of Zero Loss Representations}
\label{sec:geometryofzeroloss}

In \eqref{eq:defsiglip},
$\losssiglip(\{(U_i,V_i)\}_{i = 1}^N;t,b)\ge 0$ holds for any inputs because $\log(1 + e^\kappa)\ge 0$ for any $\kappa\in \mathbb{R}.$ Hence, global minimizers are any choice of representations and parameters $\{(U_i,V_i)\}_{i = 1}^N;t,b\in (\mathbb{S}^{d-1})^{\otimes N}\times (\mathbb{S}^{d-1})^{\otimes N}\times [0, +\infty]\times [-\infty, \infty]$ leading to a zero loss. We characterize such configurations fully in the following theorems. The proofs are simple and delayed to \cref{appendix:geometryofzeroloss}.
% We call \emph{zero-loss} any configuration of representations $\{(U_i,V_i)\}_{i = 1}^N$ for which one can take a sequence of $(t^{(s)}, b^{(S)})\in [0, +\infty)\times (-\infty, \infty)$ such that $\lim_{s\longrightarrow + \infty}\losssiglip(\{(U_i,V_i)\}_{i = 1}^N;t^{(s)},b^{(s)}) = 0.$

\begin{theorem}
[All Global Minima are $(\margin,\relativebias)$-Constellations]
\label{thm:convergencetozeroloss}
Suppose that any iterative algorithm produces a sequence $\{U^{(s)}_i\}_{i =1}^N, \{V^{(s)}_i\}_{i =1}^N,t^{(s)}>0, b^{(s)}$ for $s = 1,2, \ldots$ such that 
$$
\lim_{s\longrightarrow + \infty} \losssiglip(\{U^{(s)}_i\}_{i = 1}^N, \{V^{(s)}_i\}_{i=1}^N;t^{(s)},b^{(s)})= 0.
$$
Then, there exists some subsequence indexed by $(s_r)_{r=1}^{+\infty}$ such that
\begin{equation}
\begin{split}
    \lim_{r \longrightarrow +\infty} U_i^{(s_r)} = U_i,\quad  
    \lim_{r \longrightarrow +\infty} V_i^{(s_r)} = V_i \text{ for all }i,\quad
    \lim_{r \longrightarrow +\infty} \frac{b^{(s_r)}}{t^{(s_r)}} = \relativebias,
\end{split}
\end{equation}
and there exists some $\margin\ge 0$ such that $\{(U_i,V_i)\}_{i= 1}^N,\margin,\relativebias$ satisfy \eqref{eq:mrconstellation}.
\end{theorem}

% The minima have a surprisingly simple structure. They satisfy $\langle U_i, V_i\rangle \ge \margin + \relativebias,$ and $\langle U_i, V_j\rangle \le -\margin + \relativebias.$ 
% We call $\margin$ the \emph{margin} and $\relativebias$ the \emph{relative bias} and the overall configuration $\{(U_i, V_i)\}_{i = 1}^N, \margin,\relativebias$  a \emph{zero-loss configuration}. 
% Clearly, such $\margin, \relativebias$ exist if and only if the following inner product separation condition holds:
% \begin{equation}
% \label{eq:linearseparability}
%     \tag{IP Separation}
%     \min_{i}\langle U_i, V_i\rangle \ge \max_{i\neq j}\langle U_i, V_j\rangle.
% \end{equation}

\begin{theorem}[All $(\margin,\relativebias)$-Constellations Are Global Minimizers]
\label{thm:ipimplieszeroloss}
Suppose that $\{(U_i,V_i)\}_{i=1}^N \in \mathbb{S}^{d-1}$ satisfies \eqref{eq:mrconstellation} for some $\margin>0.$ If we set $b = \relativebias \times t$, then
$$
\lim_{t\longrightarrow  + \infty}
\loss(\{U_i\}_{i=1}^N,\{V_i\}_{i=1}^N;t, \relativebias \times t) = 0\,.
$$    
Moreover, for
$\margin^* \coloneqq \frac{1}{2}(\min_{i}\langle U_i, V_i\rangle - \max_{i\neq j}\langle U_i, V_j\rangle), \relativebias^* \coloneqq \frac{1}{2}(\min_{i}\langle U_i, V_i\rangle + \max_{i\neq j}\langle U_i, V_j\rangle),$  
$$
\inf_{b}
\losssiglip(\{U_i\}^N_{i=1}, \{V_i\}_{i=1}^N;t, b)= e^{- t\margin^* + o(t)}
$$
and is achieved when $b = \relativebias^*\times t + o(t).$
\end{theorem}

\begin{table}[h!]
\centering
\caption{Margin and Relative bias corresponding to 5th-percentile positive and 95-th percentile negative pairs for different SigLIP models. We highlight that the two largest so400m models have a substantially different relative bias than the rest of the models. Likewise, the margin is perfectly correlated with the embedding dimension -- bigger models have bigger margin.}
\label{tab:siglip_results_5percent}
\small % Use a smaller font size for the table
\setlength{\tabcolsep}{4pt} % Reduce the space between columns
\begin{tabular}{@{}lccccc@{}} % 1 'l' + 4 'c' = 5 columns
\toprule
\textbf{Model} & \textbf{5\% Positive Pairs} & \textbf{95\% Negative Pairs} & \textbf{Margin} & \textbf{Relative Bias} & \textbf{Dimension}\\
\midrule
\color{red}siglip-so400m-patch14-384 \color{black} & 0.0769 & 0.0486 & 0.0142 & 0.0627 & 1152 \\
\color{red}siglip-so400m-patch14-224 \color{black} & 0.0747 & 0.0483 & 0.0132 & 0.0615 & 1152 \\
siglip-large-patch16-256  & 0.0400 & 0.0151 & 0.0124 & 0.0276 & 1024 \\
siglip-large-patch16-384  & 0.0353 & 0.0120 & 0.0117 & 0.0237 & 1024 \\
siglip-base-patch16-512   & 0.0409 & 0.0170 & 0.0120 & 0.0289 & 768 \\
siglip-base-patch16-384   & 0.0408 & 0.0173 & 0.0118 & 0.0290 & 768 \\
siglip-base-patch16-256   & 0.0413 & 0.0200 & 0.0106 & 0.0306 & 768 \\
siglip-base-patch16-224   & 0.0383 & 0.0181 & 0.0101 & 0.0282 & 768 \\
\bottomrule
\end{tabular}
\end{table}
This characterization is for global minima in the case of zero loss, which may seem too idealistic. However, it turns out that practically trained models are (up to a small error) also Constellations. In \cref{tab:siglip_results_5percent}, we provide the optimal relative bias and margin after removing 5\% outliers from the positive and negative pairs.

In \cref{sec:infonce}, we also provide a similar characterization for the global minimizers of the InfoNCE loss, by proving  \cref{eq:InfoNCEGlobMinima}. It turns out that $(\margin,\relativebias)$-constellations are also global minimizers for the triplet loss, which is another popular contrastive training objective~\cite{Schroff2015tripletloss}. On the other hand, the global minimizers of InfoNCE and triplet loss completely coincide. See \cref{appendix:geometryofzeroloss}.

\cref{thm:ipimplieszeroloss} shows not only that any $(\margin,\relativebias)$-constellation is a global minimizer, but also that \emph{the optimal margin $\margin^*$ characterizes the speed of convergence of the loss to zero.}

\paragraph{Any $(\margin,\relativebias)$-Constellation yields perfect retrieval}follows as a corollary of \cref{thm:convergencetozeroloss}. In the image-text retrieval task, one is given an image (respectively text) and has to produce the text (respectively image) that best matches it. In our mathematical model, this corresponds to producing $U_i$ on input $V_i$ (and $V_i$ on input $U_i$).

\begin{corollary}[Nearest Neighbor Search Yields Perfect Retrieval]
\label{cor:perfectretrieval}
Suppose that $\{(U_i,V_i)\}_{i = 1}^N$ is a zero loss configuration. Then, a nearest neighbor of $U_i$ among $\{V_j\}_{j = 1}^N$ is $V_i.$ If, furthermore, the margin $\margin$ is strictly positive, this neighbor is unique.  
\end{corollary}
In practice, retrieval is often performed via \emph{approximate} nearest neighbor search \cite{xiong2021approximate,khattab20colbert,macdonald21ann} as this approach has significant computational efficiency advantages. Hence, representations more robust to approximation errors are more desirable. Since $\min_{i}\langle U_i, V_i\rangle - \max_{i\neq j}\langle U_i, V_j\rangle\ge 2\margin $ when  \eqref{eq:mrconstellation} is satisfied, representations with a larger margin are more robust. The importance of margin on retrieval has been empirically observed and exploited in several empirical works
\cite{liu16large,deng22arcface}. \cref{cor:perfectretrieval} is for perfect zero-loss constellations. A more robust version also holds, which is closer to practice due to the fact that practical models \emph{are not trained to zero loss} (and cannot be, both due to computational limitations and mislabeled data). We note that in the basic version of the proposition, one can ignore batch size and take $B = N.$ We also include the effect of batch size since this is how models are trained in practice.  

\begin{proposition}[Robustness of Retrieval via Nearest Neighbor Search]
\label{prop:zeroloss}
Let the embedded dataset be $\{(U_i, V_i)\}_{i = 1}^N.$ Suppose that for inverse temperature and bias $t>0,b$ and some $\xi\in [0,1],$ and some batch size $N> B>\sqrt{N},$
it holds that:\footnote{We denote by $[N]^{\times B}$ the uniform distribution over $B$-tuples in $[N]$ of distinct indices.}
\begin{align*}
& \expect_{(a_j)_{j = 1}^B\sim [N]^{\times k}}\Big[ \sum_{j = 1}^B \log \Big(1 + \exp(- t \langle U_{a_j}, V_{a_j}\rangle + b)\Big)\\
& \qquad\qquad\qquad\qquad\qquad\qquad+ 
    \sum_{i \neq j} \log \Big(1 + \exp(t \langle U_{a_i}, V_{a_j}\rangle - b)\Big)\Big]\le \xi \log 2.
\end{align*}
Then, for at least a $1-\frac{N\xi}{B(B-1)}$ fraction of the values $U_i$ (respectively, $V_i$), a nearest neighbor search returns $V_i$ (respectively, $U_i$).
\end{proposition}
The proof is delayed to \cref{appendix:robustretrieval}. Note that while $N> B>\sqrt{N}$ is restrictive, it is relevant to models trained with massive compute. For example, in \cite{zhai23siglip}, the authors run models with batch sizes up to $64000$ which makes the statement meaningful for datasets of size as large as $10^{10}.$ 

\subsection{Constructions of \texorpdfstring{$(d,\margin,\relativebias)$}{m,rb}-Constellations And Cardinality Bounds}
\label{sec:cardinalitybounds}
Our results so far are vacuous if no $(\margin,\relativebias)$-Constellations exist. In this section, we show a generic construction which is largely motivated by the Double-Constant Embedding Model of \cite{lee24analysissiglip} but replaces the simplex with a \emph{spherical code}. For $\alpha\in[-1,1)$ and $d\in \mathbb{N},$ a $(d,\alpha)$-spherical code is a collection of vectors $X_1, X_2, \ldots, X_N\in \mathbb{S}^{d-1}$ such that $\langle X_i, X_j\rangle \le \alpha$ for all $i\neq j$ \cite[(52)]{ConwayJH2013Spla}. 
In particular, any $(d,\alpha)$ code is a $(d, \frac{1-\alpha}{2},\relativebias = \frac{1+\alpha}{2})$-constellation and vice-versa. 
This implies that any construction of spherical codes immediately implies a construction of $(\margin,\relativebias)$-Constellations when $\margin + \relativebias= 1.$ Spherical codes are a well-studied object in combinatorics and many constructions exist depending on $\alpha$ (see \cite{ConwayJH2013Spla} and references therein). The following construction shows that we can extend to the case when $\margin + \relativebias\neq 1.$

\begin{example}[Construction of $(\margin,\relativebias)$-Constellations]
\label{thm:construction2modalitieswithgap}
Consider any $(d-2,\margin,\relativebias)$-constellation $\{(U_i, V_i)\}_{i = 1}^N.$  Then, for any $\delta,\phi\in [0,1)$ such that $\delta^2 + \phi^2\le 1,$ the following vectors form a $(d, \margin', \relativebias')$-constellation with $\margin' = \delta^2 \margin$ and $ \relativebias' = \delta^2\relativebias +\phi^2 - (1 - \delta^2 - \phi^2)$:
    \begin{equation}
        \begin{split}
            & U'_i = (\delta U_i, \phi, \sqrt{1-\delta^2 - \phi^2})\qquad \text{ and }\qquad V'_i = (\delta V_i, \phi, - \sqrt{1-\delta^2 - \phi^2}).
        \end{split}
    \end{equation}
\end{example}

\begin{wrapfigure}{r}{0.24\textwidth}
\vspace*{-0.6cm}
\includegraphics[width = \linewidth]{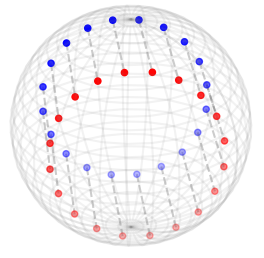}
\vspace*{-0.7cm}
\caption{\small {3D visualization of configurations in Example \ref{example:dminus3nonsep}, obtained by minimizing sigmoid loss with Adam; each $(U_i, V_i)$ pair is a reflection across a hyperplane.}}
\label{fig:basicfigure}
\vspace*{-0.7cm}
\end{wrapfigure}

This construction shows that $(\margin,\relativebias)$-Constellations not only exist but constitute a rich family. One can in fact construct them from any locked $(d-2)$-dimensional embedding $\{X_i\}_{i = 1}^N$ as long as $X_i\neq X_j$ for $i\neq j$ which is the basis of our algorithm for synchronizing with a locked encoder in Observation \ref{obs:lockedapadapterimplicit}. Recall that the margin impacts the robustness of the representation for retrieval. Thus, it is of both practical and theoretical interest to analyze how large the margin could be for a given dimension $d$ and sample size $N.$ Construction~\ref{thm:construction2modalitieswithgap} immediately gives a recipe for this based on spherical code bounds.

 That is, let $ \sizesphericalcode(d,\alpha)$ be the largest possible size of a $(d,\alpha)$-spherical code. Let $\logsizesphericalcode(\alpha) = \lim\inf_{d\longrightarrow + \infty}\frac{\log \sizesphericalcode(d,\alpha)}{d}.$
 Determining $\sizesphericalcode,\logsizesphericalcode$ is a well studied problem in coding theory %with its origins in biology \cite{tammes1930pollens} 
 \cite{ConwayJH2013Spla}. We similarly define the numbers $\sizeconstellation(d,\margin,\relativebias)$ and 
$\logsizeconstellation(\margin,\relativebias)$ for the largest $(\margin,\relativebias)$-constellations in dimension $d.$  Construction~\ref{thm:construction2modalitieswithgap}, together with the classical bound $\logsizesphericalcode(\alpha)\ge \log_2\frac{1}{\sqrt{1-\alpha^2}}$ due to Shannon \cite{shannon59lb} and Wyner \cite{wyner68sphericallb}  implies:

\begin{theorem}[Lower Bound on the Size of Constellations]
\label{thm:lowerbondsviasphericalcodes}
Suppose that $\margin \ge 0, \relativebias\in [-1,1]$ satisfy $\margin + \relativebias< 1$ and 
$3\margin < 1 + \relativebias.$ Then, there exist $(\margin,\relativebias)$-constellations of size exponential in dimension and furthermore
$$
\logsizeconstellation(\margin, \relativebias)\ge 
\logsizesphericalcode\Big(\frac{1 + \relativebias - 3\margin}{1 + \relativebias + \margin}\Big)\ge 
- \frac{1}{2}\log_2
\Big(1 - \Big(\frac{1 + \relativebias - 3\margin}{1 + \relativebias + \margin}\Big)^2\Big).
$$
\end{theorem}
\begin{proof} Let $\alpha\coloneqq \frac{1 + \relativebias - 3\margin}{1 + \relativebias + \margin}\in [0,1].$ Let $X_1, X_2, \ldots, X_N$ be an $\alpha$-spherical code in dimension $d-2$ of size $\exp\Big((d-2)(\logsizesphericalcode(\alpha) + o(1))\Big) = 
\exp\Big(d(\logsizesphericalcode(\alpha) + o(1))\Big).
$ Choose $\phi, \delta$ as follows:
\begin{align*}
    & \delta^2 = \frac{2\margin}{1 - \alpha},\qquad\qquad \phi^2 = \frac{2\relativebias + 2 - \delta(3 + \alpha)}{4}.
\end{align*}
One can easily check that the inequalities $\margin + \relativebias\le 1$ and 
$3\margin \le 1 + \relativebias$ imply that these values are well-defined in the sense that $\delta^2 >0, \phi^2>0$ and, furthermore, $\delta^2 + \phi^2\le 1.$ Now, we can apply Construction \ref{thm:construction2modalitieswithgap} with $U_i = V_i = X_i$ and $\delta, \phi$ and conclude the desired result.
\end{proof}

The conditions $\margin + \relativebias\le 1$ and 
$3\margin \le 1 + \relativebias$ are not a virtue of our construction, but turn out to be necessary via an argument resembling Rankin's proof that among any $k+1$ vectors in $\mathbb{S}^{d-1},$ there exist two with inner product at least $-\frac{1}{k}$ \cite{Rankin_1955}. We actually manage to prove a lower bound for a more general set of configurations than constellations.

\begin{theorem}[Upper Bounds on Margin via Relative Bias]
\label{thm:upperboundsmargin}
Suppose that $\{(U_i, V_i)\}_{i = 1}^N$ satisfy that 
$\frac{1}{N}\sum_i \langle U_i, V_i\rangle\ge m + \relativebias$ and $\frac{1}{N(N-1)}\sum_{i\neq j} \langle U_i, V_j\rangle\le -m + \relativebias$ (in particular, this holds for any $(\margin,\relativebias)$-constellation).   Then, it also holds that
$$
\margin  +\relativebias \le 1 \quad\text{ and }\quad
3\margin\le 1 + \relativebias +o(1)\,.
$$
\end{theorem}
\begin{proof} The inequality $\margin  +\relativebias \le 1$ is trivial since $\margin + \relativebias\le \max \langle U_1, V_1\rangle \le\max  \|U_1\|_2 \times \|V_1\|_2\le 1$ by \eqref{eq:mrconstellation} and Cauchy-Schwarz. For the second inequality, we use the following fact from \cite{lee24analysissiglip}. For any unit vectors $\{(U_i, V_i)\}_{i = 1}^N,$ it holds that
$$
\frac{1}{N^2}\sum_{i\neq j}\langle U_i, V_j\rangle \ge \frac{N-2}{2N^2}\sum_i \langle U_i, V_i\rangle - \frac{1}{2}.
$$    
Using $\langle U_i, V_j\rangle\le \relativebias - \margin, \langle U_i, V_i\rangle\ge \relativebias + \margin$ gives 
$(3 - \frac{4}{N})\margin\le 1 + \relativebias.$
\end{proof}

We also provide upper bounds on the size of a constellation given the margin $\margin$. This can be used to inform the size of the embedding space given the number of pairs $(U_i, V_i)$ we want to embed in it.
% \begin{theorem}[Upper Bound on the Size of Constellations] 
% \label{thm:upperboundconst}
% Suppose that $\{(U_i,V_i)\}_{i=1}^N,$ is a $(\margin,\relativebias)$-constellation with $\margin \le \min({1+\relativebias\over 3}, 1-\relativebias)$, then for \[\theta\in(0,\frac{\pi}{2})\;\text{ defined by }\;
% \cos\theta=\sqrt{\frac{1-3m+\relativebias}{1+m+\relativebias}},
% \]
% the cardinality $N$ obeys the asymptotic upper bound
% \[
% N\;\le\;\exp\!\Bigl(-d\,\log\sin\theta+o(d)\Bigr),
% \qquad d\to\infty.\]
% Equivalently,
% \[
% \logsizeconstellation(\margin,\relativebias) \le -\log \sin \theta.
% \]
% \end{theorem}

\begin{theorem}[Upper Bound on the Size of Constellations] 
\label{thm:upperboundconst}
    Suppose that $\{(U_i,V_i)\}_{i=1}^N,$ is a $(\margin,\relativebias)$-constellation for some $\margin\ge 0, \relativebias\in [-1,1]$ which satisfy 
$\margin + \relativebias\le 1$ and 
$3\margin \le 1 + \relativebias.$
Then,
$$
N\;\le\;\exp\!\Big(-d\,\frac{1}{2}\log\Big(
1 - \frac{1+\relativebias - 3\margin}{1 + \relativebias + \margin}
\Big)+o(d)\Big).$$
Equivalently,
$$
\logsizeconstellation(\margin,\relativebias) \le -\frac{1}{2}\log\Big(
1 - \frac{1+\relativebias - 3\margin}{1 + \relativebias + \margin}
\Big).
$$
\end{theorem}

The proof is delayed to \cref{sec:upperboundconstellationsize}. We plot the upper and lower bounds from this section in Figure~\ref{fig:upperlowerboundsplot} for $\relativebias = 0$ and $m = 0.1$. In \cref{appendix:LRH}, we note that the proof also illustartes a connection with the linear representation hypothesis, e.g. \cite{park24LRH}.

\begin{figure}[!h]
    \centering
    \includegraphics[width=0.9\linewidth]{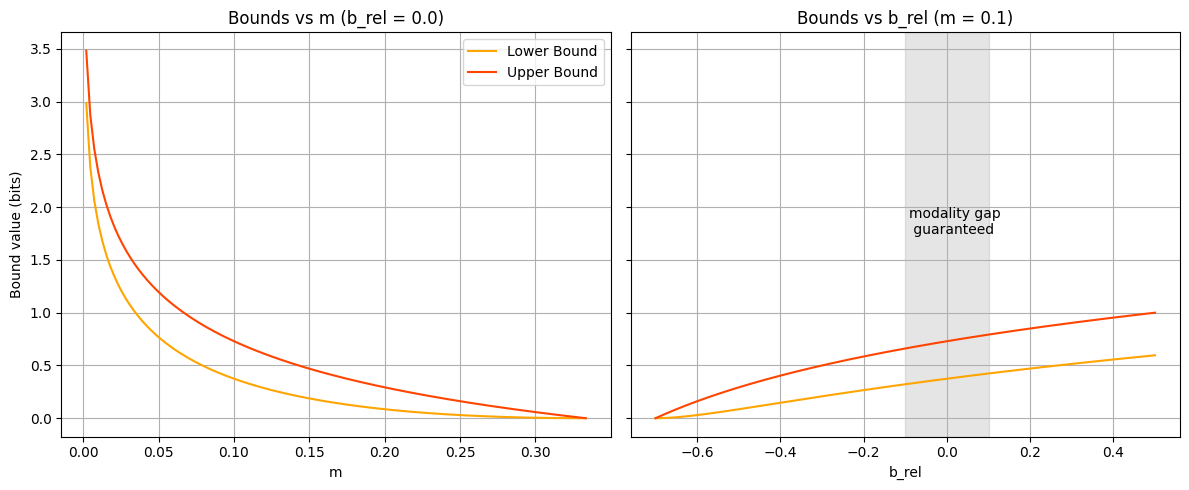}
    \caption{Upper and Lower Bounds from \cref{thm:upperboundconst} and \cref{thm:lowerbondsviasphericalcodes} for fixed $\relativebias = 0$, $m=0.1$.}
    \label{fig:upperlowerboundsplot}
\end{figure}

% \nbyp{I think the paper should be slightly rewritten to emphasize more these last two theorems. Specifically, we can define SigLIP loss early and introduce its rb version, then mention that asymptotic configs have (m,rb) property and then state the main theorem that the set of achievable (m,rb) is blah.}

\begin{wrapfigure}{r}{.5\textwidth}
\vspace*{-1cm}
\resizebox{\linewidth}{!}{
\begin{tikzpicture}
  \begin{axis}[
    xlabel={Relative Bias $\relativebias$},
    ylabel={Margin $\margin$},
    xlabel style={at={(axis description cs:0.2,-0.1)}, anchor=north},
    ylabel style={at={(axis description cs:-0.1,0.1)}, anchor=south},
    xmin=-.1, xmax=.1,
    ymin=0, ymax=.1,
    axis lines=middle,
    enlargelimits=false,
    width=13cm,
    height=6.5cm,
    xtick={-.1, -0.05, 0, 0.05, .11},
    ytick={0, .025, 0.05, .075,.1},
    grid=both,
    clip=true
  ]

  % Green triangle with transparency
  \addplot [
    fill=green!50,
    draw=black,
    opacity=0.6
  ] coordinates {
    (-.1, 0)
    (-.1,.3)
    (0.1,0.3667)
    (.1,0)
    (-.1,0)
  };

  % Modality Gap Region with pattern
  \addplot [
    pattern=north east lines,
    pattern color=black!30,
    draw=none
  ] coordinates {
    (-.1,.3)
    (-.1, .1)
    (0,0)
    (.1, .1)
    (0.1,0.3667)
    (-.1,.3)
  };

  \addplot
  [
  only marks,
  mark=triangle*,
  mark options={fill=red},
  mark size=4pt
] coordinates {(0.068, 0.07)};

  \addplot
  [
  only marks,
  mark=triangle*,
  mark size=4pt
] coordinates {(0.0322, 0.0627)};

  \addplot
  [
  only marks,
  mark=triangle*,
  mark size=4pt
] coordinates {(0.0323, 0.0642)};

  \addplot
  [
  only marks,
  mark=triangle*,
  mark size=4pt
] coordinates {(0.0332, 0.069)};

  \addplot
  [
  only marks,
  mark=triangle*,
  mark options={fill=red},
  mark size=4pt
] coordinates {(0.067, 0.069)};

  \addplot
  [
  only marks,
  mark=triangle*,
  mark size=4pt
] coordinates {(0.035, 0.065)};

  \addplot
  [
  only marks,
  mark=triangle*,
  mark size=4pt
] coordinates {(0.0324, 0.065)};

  \addplot
  [
  only marks,
  mark=triangle*,
  mark size=4pt
] coordinates {(0.02876, 0.067)};
  \end{axis}
\end{tikzpicture}}
\vspace*{-.6cm}
\caption{\small{Mean inner products of positive pairs $\langle U_i, V_i\rangle$ versus mean inner products of negative pairs $\langle U_i, V_j\rangle$ from the ImageNet validation dataset. }}
\label{fig:realsiglippoints}
\vspace*{-.7cm}
\end{wrapfigure}

We note that in \cite[Theorem 4]{robinson2021contrastive}, the authors find a different connection between spherical codes and minimizers of the InfoNCE loss. In the concurrent works \cite{weller2025theoretical,snoeck2025compressibilitybarriersneighborhoodpreservingdata}, the authors show a different dimension lower-bound for existence of vector embeddings for top-$k$ retrieval.

We end with zooming into the Figure \ref{fig:possiblemarginbiasregion} and plotting the performance of several trained siglip models from Hugging Face. We observe two clusters -- one composed of the larger so400m models (around 1B parameters) and another cluster of smaller models (up to .4B). It is interesting to consider that the so400m are exactly on the boundary where the modality gap is guaranteed (all 8 models do satisfy the modality gap with zero misclassification error). The differnece in margins is partly explained by dimensionality -- larger dimensions correspond to larger margins. The Pearson correlation coefficient between dimension and margin is $.948$ and the Spearman coefficient is $.926.$

\begin{table}[h!]
\centering
\caption{Mean cosine similarities, margin, and relative bias for different SigLIP models.}
\label{tab:siglip_results}
\small % Use a smaller font size for the table
\setlength{\tabcolsep}{3pt} % Reduce the space between columns
\begin{tabular}{@{}lccccc@{}} % 1 'l' + 4 'c' = 5 columns
\toprule
\textbf{Model} & \textbf{Mean Pos. Pairs} & \textbf{Mean Neg. Pairs} & \textbf{Margin} & \textbf{Relative Bias} &\textbf{Dimension}\\ 
\midrule
\color{red}siglip-so400m-patch14-384 \color{black} & 0.1376 & -0.0015 & 0.0695 & 0.0680 & 1152 \\
\color{red}siglip-so400m-patch14-224 \color{black} & 0.1365 & -0.0022 & 0.0694 & 0.0672 & 1152 \\
siglip-large-patch16-256  & 0.1023 & -0.0359 & 0.0691 & 0.0332 & 1024 \\
siglip-large-patch16-384  & 0.0958 & -0.0384 & 0.0671 & 0.0287 & 1024 \\
siglip-base-patch16-256   & 0.1004 & -0.0294 & 0.0649 & 0.0355 & 768 \\
siglip-base-patch16-512   & 0.0971 & -0.0322 & 0.0646 & 0.0324 & 768 \\
siglip-base-patch16-384   & 0.0966 & -0.0319 & 0.0642 & 0.0324 & 768 \\
siglip-base-patch16-224   & 0.0950 & -0.0305 & 0.0627 & 0.0322 & 768 \\
\bottomrule
\end{tabular}
\end{table}

\subsection{The Modality Gap in SigLIP}
\label{sec:modalitygap}
The construction in Example \ref{thm:construction2modalitieswithgap} satisfies the modality gap property -- when $\delta >0,$ the representations of the two modalities are separated by a hyperplane (orthogonal to the last coordinate). This phenomenon appears not only in our construction, but has been observed empirically on synchronized text and image embeddings in CLIP \cite{liang2022mindthemodalitygap,fahim2025its} and
in SigLIP by us in Figure \ref{fig:modalitygasiglip}.
We show a rigorous justification for this.
\begin{theorem}[Modality Gap in Zero-Loss Configurations]
\label{thm:sperablemodalities}
Suppose that
$N\ge {d+2}$ and 
$\{(U_i,V_i)\}^{N}_{i = 1}$ are such that
$\langle U_i, V_i\rangle > 0$ for all $i,$ 
$\langle U_i, V_j\rangle < 0$ for all $i\neq j.$ This happens for example, when $\margin> |\relativebias|$ in a $(\margin,\relativebias)$-Constellation and when $|\relativebias(i)|<\margin$ for InfoNCE in \eqref{eq:InfoNCEGlobMinima}.
Then, there exists some $h\in \mathbb{S}^{d-1}$ such that:
\begin{align}
    & \langle h, U_i\rangle > 0 \quad\quad\text{ for all }i, \label{conditiononU}\\
    & \langle h, V_j\rangle < 0 \quad\quad \text{ for at least $N-d$ values of $j.$} \label{nearconditiononV}
\end{align}
\end{theorem}
The fact that the condition \eqref{nearconditiononV} is satisfied for at least $N-d$ values of $j$ instead of all $N$ is rather minor in practice. As mentioned, in SigLIP2 \cite{tschannen2025siglip2}, $N \approx  10^{10}$ and $d \approx 10^3.$ Thus, our result shows the modality gap holds for all but $.0000001\%$ of the text embeddings. We note that the theorem is essentially tight -- in Example \ref{example:dminus3nonsep}, we show an example in which $d-1$ of the vectors $V_j$ cannot be separated from the vectors $U_i.$
We also note that $\margin> |\relativebias|$ is also plausible as practically trained models such as SigLIP2 have a small relative bias of magnitude less than $0.1$ \cite{siglip2demo}. 
\begin{proof}[Proof Sketch] The full proof is in \cref{appendix:modalitygapproofs}, where we also analyze further properties of configurations satisfying \eqref{conditiononU} and \eqref{nearconditiononV}.
Here we give a sketch.
First, we use Helly's theorem (\cref{thm:helly}) to show that the convex sets $\{x:\langle x, U_i\rangle >0\}_{i = 1}^N$ have a non-empty intersection and, hence, there exists some $h\in \mathbb{S}^{d-1}$ such that $\langle h, U_i\rangle >0$ for each $i.$ Then, we use the hyperplane separation theorem (\cref{thm:hyperplaneseparation}) to show that the projection $\bar{h}$ of $h$ on the convex cone defined by $U_1, U_2, \ldots, U_N$ also has this property. Finally, we use Caratheodory's theorem (\cref{thm:caratheodoery}) to show that $\bar{h}$
has a positive inner product with all the vectors $U_i$ and
is in the convex cone of at most $d$ of the vectors $U_j.$ This implies that $\bar{h}$ has a negative inner product with all the other $N-d$ vectors $V_k.$
\end{proof}

\subsection{Experiments: Sigmoid Loss with Explicit Relative Bias Parameterization}
\label{sec:parameterization}
Due to the importance of the relative bias parameters for global minima of sigmoid loss, we propose a parameterization that explicitly captures this dependence.
\begin{definition}[Parameterization with Explicit Relative Bias] 
\label{def:relativebiasparam}
The relative bias parametrization of the sigmoid loss for encoder $f_\theta, g_\phi$ over data pairs $\{(X_i, Y_i)\}_{i = 1}^N$ 
with $U_i = f_\theta(X_i), V_i = g_\phi(Y_i)$
is
\begin{equation}
    \label{eq:defsigmoidrelativebias}
    \begin{split}
    & \losssigliprb(\theta, \phi;t,\relativebias)\\ 
    &= 
    \sum_{i = 1}^N \log \Big(1 + \exp(-t\langle U_i, V_i\rangle  +t \relativebias)\Big) + 
    \sum_{i \neq j} \log \Big(1 + \exp(t\langle U_i, V_j\rangle - t \relativebias)\Big).
    \end{split}
\end{equation}
\end{definition}

\begin{wrapfigure}{r}{0.4\textwidth}
 \vspace*{-1.2cm}
    \includegraphics[width=\linewidth]{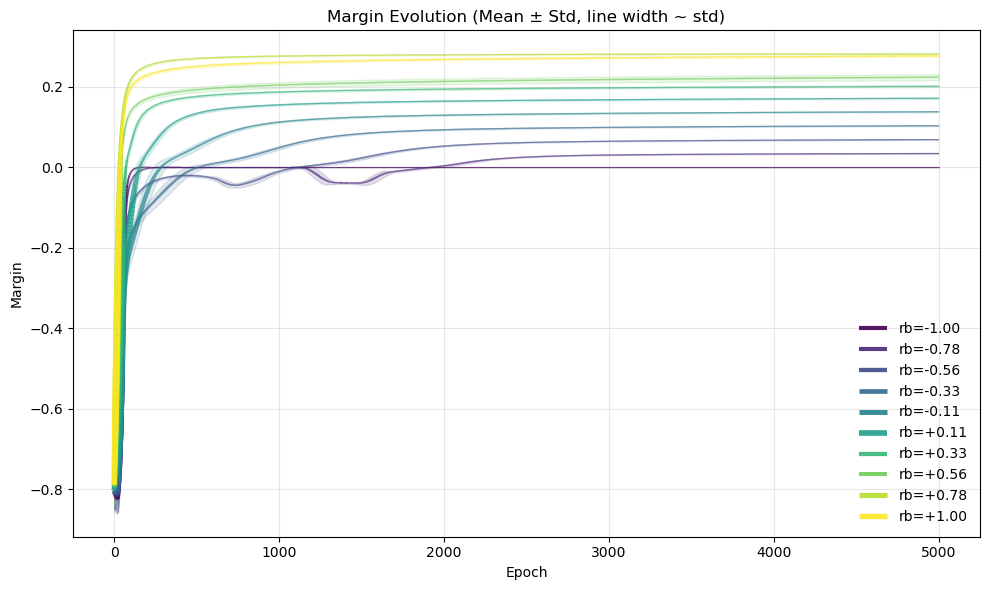}
   
    \caption{\small{Evolution of margins when training with different fixed relative biases, average over 100 iterations.} } 
     \vspace*{-.2cm}
\label{fig:fixedrb_margin_evolution}
\end{wrapfigure}

Clearly, $\losssigliprb(\theta, \phi;t,\relativebias) = \losssiglip(\theta, \phi;t,\relativebias\times t)$ so the loss functions are the same. However, we show that running Adam \cite{KingmaB14adam} on $\losssigliprb$ yields faster convergence -- see Figures \ref{fig:lockedrepresenattion} and \ref{fig:multiplemodalities}. It also provides the additional flexibility to freeze the relative bias to a desired value and only train inverse temperature. This may be important since we observe that in practice relative bias converges to 0 when not frozen. For example, in SigLIP2 \cite{siglip2demo} with the B/16 model with resolution $384\times 384,$ we have learned parameters $t \approx 117.8, b\approx -12.9,$ so $\relativebias\approx -0.11.$ We give further experimental evidence for this in \cref{appendix:biasparametrization}. 
Thus, we propose using our parameterization $\losssigliprb$ in practice over $\losssiglip.$

\begin{wrapfigure}{r}{.5\textwidth}
\vspace*{-1cm}
\resizebox{\linewidth}{!}{
\begin{tikzpicture}
  \begin{axis}[
    xlabel={Relative Bias $\relativebias$},
    ylabel={Margin $\margin$},
    xlabel style={at={(axis description cs:0.5,-0.1)}, anchor=north},
    ylabel style={at={(axis description cs:-0.1,0.5)}, anchor=south},
    xmin=-1, xmax=1,
    ymin=0, ymax=1,
    axis lines=middle,
    enlargelimits=false,
    width=13cm,
    height=6.5cm,
    xtick={-1, -0.5, 0, 0.5, 1},
    ytick={0, 0.1, 0.2, 0.3, 0.4, 0.5,.6, .7, .8,.9,1},
    grid=both,
    clip=true
  ]

  % Red background with transparency
  \addplot [
    fill=red!30,
    draw=none,
    opacity=0.3
  ] coordinates {
    (-1,0)
    (1,0)
    (1,1)
    (-1,1)
    (-1,0)
  };

  % Green triangle with transparency
  \addplot [
    fill=green!50,
    draw=black,
    opacity=0.6
  ] coordinates {
    (-1,0)
    (0.5,0.5)
    (1,0)
    (-1,0)
  };

  % Modality Gap Region with pattern
  \addplot [
    pattern=north east lines,
    pattern color=black!30,
    draw=none
  ] coordinates {
    (-1/4,1/4)
    (0,0)
    (0.5,0.5)
    (-1/4,1/4)
  };

  % Labels
  \addplot
  [
  only marks,
  mark=square*,
  mark size=2pt
] coordinates {(-1, 0)};
 \node[align=center] at (axis cs:-1,0.05) { {$-1$}
  %\cref{thm:sperablemodalities}
  };

    % Labels
  \addplot
  [
  only marks,
  mark=square*,
  mark size=2pt
] coordinates {(-.78, 0.034)};
 \node[align=center] at (axis cs:-.78, 0.084) { {$-7/9$}
  %\cref{thm:sperablemodalities}
  };

  \addplot
  [
  only marks,
  mark=square*,
  mark size=2pt
] coordinates {(-.56, 0.068)};
 \node[align=center] at (axis cs:-.56, 0.138) { {$-5/9$}
  %\cref{thm:sperablemodalities}
  };

    \addplot
  [
  only marks,
  mark=square*,
  mark size=2pt
] coordinates {(-.346, 0.103)};
 \node[align=center] at (axis cs:-.346, 0.153) { {$-1/3$}
  %\cref{thm:sperablemodalities}
  };

      \addplot
  [
  only marks,
  mark=square*,
  mark size=2pt
] coordinates {(-.12, 0.13)};
 \node[align=center] at (axis cs:-.12, 0.18) { {$-1/9$}
  %\cref{thm:sperablemodalities}
  };

 \addplot
  [
  only marks,
  mark=square*,
  mark size=2pt
] coordinates {(.09, 0.17)};
 \node[align=center] at (axis cs:.09, 0.22) { {$1/9$}
  %\cref{thm:sperablemodalities}
  };

   \addplot
  [
  only marks,
  mark=square*,
  mark size=2pt
] coordinates {(.318, 0.20)};
 \node[align=center] at (axis cs:.318, 0.25) { {$ 1/3$}
  %\cref{thm:sperablemodalities}
  };

     \addplot
  [
  only marks,
  mark=square*,
  mark size=2pt
] coordinates {(.54, 0.22)};
 \node[align=center] at (axis cs:.54, 0.27) { {$5/9$}
  %\cref{thm:sperablemodalities}
  };

       \addplot
  [
  only marks,
  mark=square*,
  mark size=2pt
] coordinates {(.71, 0.28)};
 \node[align=center] at (axis cs:.71, 0.33) { {$7/9$}
  %\cref{thm:sperablemodalities}
  };

         \addplot
  [
  only marks,
  mark=square*,
  mark size=2pt
] coordinates {(.72, 0.276)};
 \node[align=center] at (axis cs:.72, 0.226) { {$1$}
  %\cref{thm:sperablemodalities}
  };
  \end{axis}
\end{tikzpicture}}
%\vspace*{-.8cm}
\caption{\small{Achieved optimal relative bias and margin when training with a fixed relative bias. The annotations correspond to the fixed relative bias.}}
\label{fig:possiblemarginbiasregion_fixedrbtraining}
\vspace*{-0.5cm}
\end{wrapfigure}
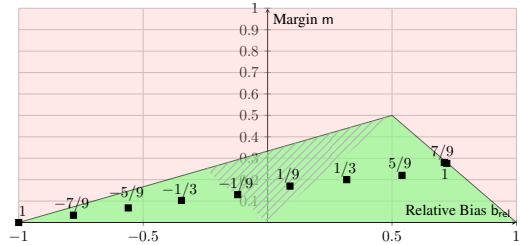

\paragraph{Fixed Relative Bias.} One functionality that this new parameterization gives us is to train models with a fixed relative bias. As expected from \cref{fig:possiblemarginbiasregion}, this in turn has an effect on the margin of the configurations. We plot in \cref{fig:fixedrb_margin_evolution} the evolution of margins (computed as $(\min_i \langle U_i, V_i\rangle - \max_{i\neq j} \langle U_i, V_j\rangle)/2$) for different fixed relative biases and in \cref{fig:possiblemarginbiasregion_fixedrbtraining} the final optimal relative biases (computed as $(\min_i \langle U_i, V_i\rangle + \max_{i\neq j} \langle U_i, V_j\rangle)/2$).

\begin{wrapfigure}{r}{0.7\textwidth}
\vspace*{-.9cm}
    \includegraphics[width=\linewidth]{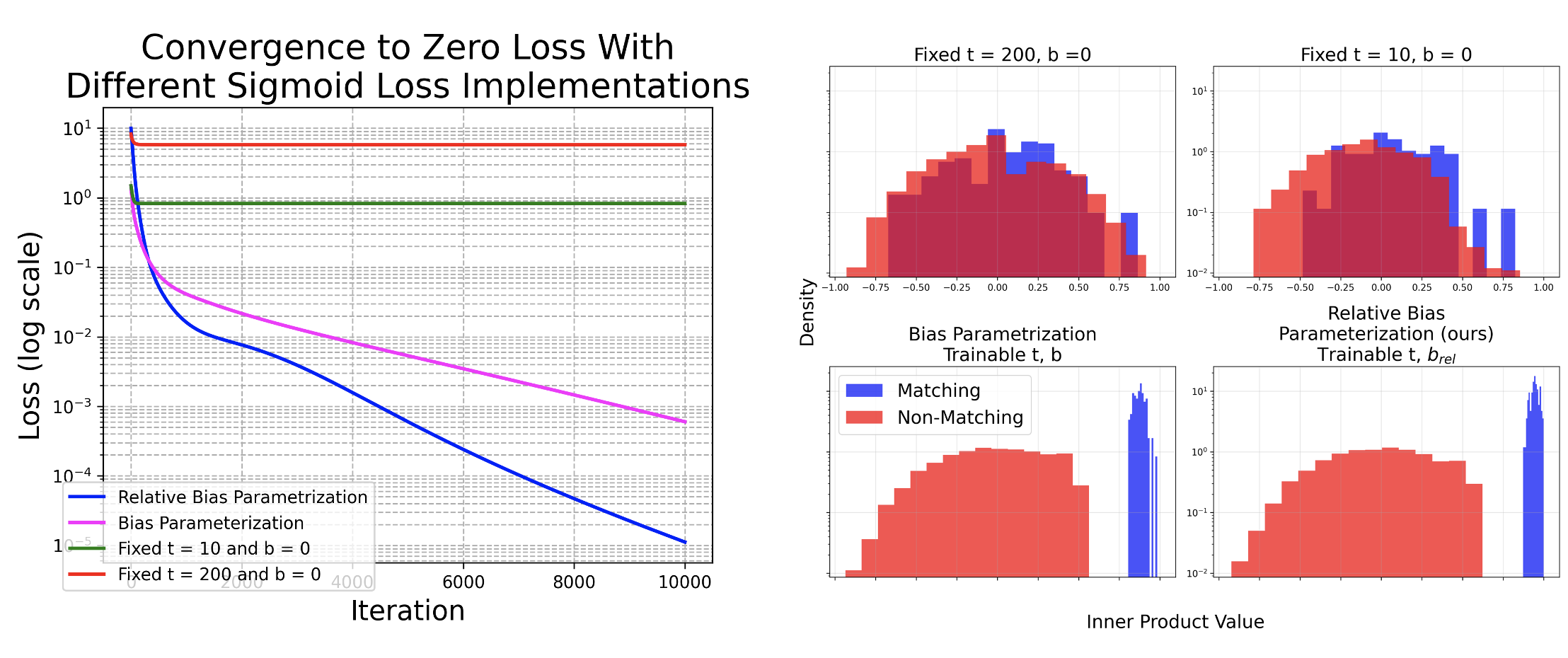}
    \vspace*{-.8cm}
    \caption{\small{Synchronizing a locked representation with different sigmoid loss functions on synthetic data. On the left, we have the evolution of the loss function. On the right, we show the distributions of non-matching inner products $\langle U_i, V_j\rangle$ for $i\neq j$ in red and matching inner products $\langle U_i, V_i\rangle$ in blue for each model.} } 
\label{fig:lockedrepresenattion}
\vspace*{-.6cm}
\end{wrapfigure}

\paragraph{Locked Encoder.} In particular, this gives a concrete recipe for training via the sigmoid loss with a fixed representation: one just adds a simple adapter $\adapter^\delta_{\mathsf{locked}}$ that transforms $X_i\longrightarrow  (\delta X_i, \sqrt{1-\delta^2})$ for the locked representation and 
$\adapter^\delta_{\mathsf{trainable}}$ that transforms $X_i\longrightarrow  (\delta X_i, -\sqrt{1-\delta^2})$ for the trainable representation.  It turns out that the relative bias parametrization captures this transformation \emph{without explicitly adding an adapter}.
\begin{observation}
\label{obs:lockedapadapterimplicit}
For any $\{(U_i, V_i)\}_{i = 1}^N$ and $\delta, \relativebias,t,$ it is the case that
$$
\losssigliprb(\{(\adapter^\delta_{\mathsf{lock}}(U_i),
\adapter^\delta_{\mathsf{train}}(V_i)\}_{i = 1}^N;t,\relativebias) = 
\losssigliprb(\{(U_i, V_i)\}_{i = 1}^N;t\delta^2,\frac{\relativebias + (1-\delta)^2}{\delta^2})\,.
$$
\end{observation}

As we can see in Figure \ref{fig:lockedrepresenattion}, the models with trainable $t,b$ (respectively $t, \relativebias$) significantly outperform the model with fixed temperature and bias. Furthermore, the convergence to zero loss is faster for $\losssigliprb$ than $\losssiglip.$ Thus, we recommend synchronizing with $\losssigliprb$ and \emph{trainable $t, \relativebias$}. 

\begin{wrapfigure}{r}{.6\textwidth}
\vspace*{-1.2cm}
\includegraphics[width = \linewidth]{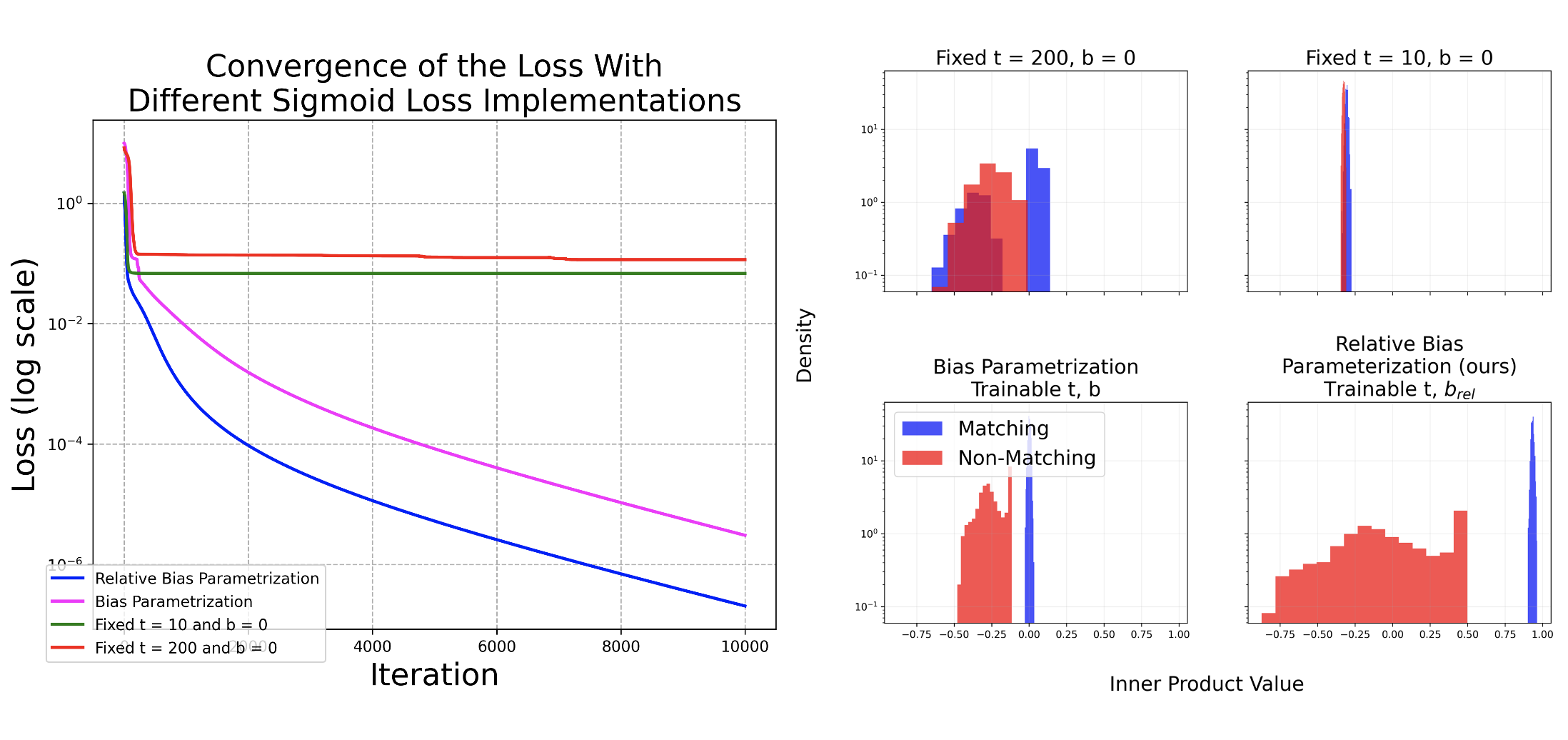}
\vspace*{-.7cm}
\caption{\small{Inner product distributions between $k = 4$ modalities synchronized with different implementations of sigmoid loss. We plot the same data as in \ref{fig:lockedrepresenattion}.}}
\label{fig:multiplemodalities} 
\vspace*{-.7cm}
\end{wrapfigure}

\paragraph{More Modalities and A New Perspective on Simplex Embeddings.}
Our discussion so far has been predominantly in the case of two modalities. To synchronize the representations 
$\{(U^{(1)}_i, \ldots, U^{(k)}_i)\}_{i = 1}^N$
of $k>2$ modalities, one typically minimizes the sum of several 
pairwise losses
\cite{tian20contrastivemultiview,girdhar23imagebind}. More formally, if $G = (V, E)$ is the \emph{synchronization graph} on vertex set $V = \{1,2,\ldots, k\}$ the different modalities, one minimizes 
\begin{equation}
\sum_{(j_1,j_2)\in E} \losssigliprb(\{(U^{(j_1)}_i, U^{(j_2)}_i)\}_{i = 1}^N;t,\relativebias).
\end{equation}

Common instances are when $G$ is the complete graph and one sums over all pairwise losses \cite{tian20contrastivemultiview}, and when $G$ is a star graph with one central modality \cite{girdhar23imagebind}.
Since the loss function $\losssiglip$ is non-negative, a configuration $\{(U^{(1)}_i, \ldots, U^{(k)}_i)\}_{i = 1}^N$ is \emph{zero-loss} if and only if there exist some $\margin, \relativebias$ such that $\{(U^{(j_1)}_i, U^{(j_2)}_i)\}_{i = 1}^N, \margin,\relativebias$ is zero loss for any $(j_1, j_2)\in E.$ In particular, $\{(U^{(1)}_i,  \ldots, U^{(k)}_i)\}_{i = 1}^N$ is zero loss if there exist some $\margin, \relativebias$ such that $\{(U^{(j_1)}_i, U^{(j_2)}_i)\}_{i = 1}^N, \margin,\relativebias$ is zero loss for all $j_1\neq j_2.$ This leads us to the following construction.

\begin{example}[Construction of Constellations]
\label{thm:constructionofconstellations}
Consider any $(d-k+1, \alpha)$-code $\{X_i\}_{i = 1}^N.$ Let $w_1, w_2, \ldots, w_k\in \mathbb{S}^{k-1}$ be the vertices of a regular $k$-simplex.
Then, for any $\delta\in [0,1),$ the following configuration is zero loss for any synchronization graph:
    \begin{equation}
        \begin{split}
            & U^{(j)}_i = (\delta X_i, \sqrt{1-\delta^2}w_j),\qquad
             \margin = \frac{\delta^2(1-\alpha)}{2},\qquad
            \relativebias = \frac{\delta^2(1-\alpha)}{2} - \frac{1-\delta^2}{k-1}.
        \end{split}
    \end{equation}
\end{example}

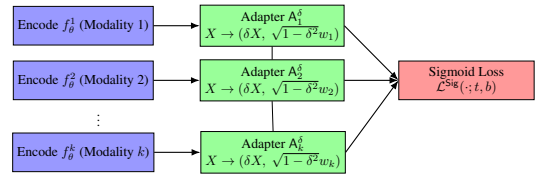
\begin{wrapfigure}{r}{0.5\textwidth}
\resizebox{\linewidth}{!}{
\begin{tikzpicture}[
  encoder/.style={draw, fill = blue!40, minimum width=3.5cm, minimum height=1cm, align=center},
  adaptor/.style={draw, fill=green!40, minimum width=3.5cm, minimum height=1cm, align=center},
  loss/.style={draw, fill=red!40, minimum width=3.5cm, minimum height=1cm, align=center},
  arrow/.style={-{Latex}, thick},
  ]
% first modality 
\node[encoder] (modality1) at (0,1.8) {Encode $f^1_\theta$ (Modality 1)};
\node[adaptor, right=1.2cm of modality1] (adapt_1) {Adapter $\adapter^\delta_1$\\ \(X \rightarrow (\delta X,\ \sqrt{1 - \delta^2} w_1)\)};

% second modality
\node[encoder, below=.4cm of modality1] (modality2) {Encode $f^2_\theta$ (Modality 2)};
\node[adaptor, right=1.2cm of modality2] (adapt_2) {Adapter $\adapter^\delta_2$\\ \(X \rightarrow (\delta X,\ \sqrt{1 - \delta^2} w_2)\)};

\node[below = .3 cm of modality2] (virtual1) { };
\node[right = .1 cm of virtual1] (virtual2) {$\vdots$};

% last modality
\node[encoder, below=.3cm of virtual1] (modalityk) {Encode $f^k_\theta$ (Modality $k$)};
\node[adaptor, right=1.2cm of modalityk] (adapt_k) {Adapter $\adapter^\delta_k$\\ \(X \rightarrow (\delta X,\ \sqrt{1 - \delta^2} w_k)\)};

% loss
\node[loss, right = 1.4cm of adapt_2] (loss_mod) {Sigmoid Loss\\
$\losssiglip(\cdot ; t,b)$
};

% Arrows
\draw[arrow] (modality1) -- (adapt_1);
\draw[arrow] (modality2) -- (adapt_2);
\draw[arrow] (modalityk) -- (adapt_k);
\draw[arrow] (adapt_1.east) -- (loss_mod.west);
\draw[arrow] (adapt_2.east) -- (loss_mod.west);
\draw[arrow] (adapt_k.east) -- (loss_mod.west);
\draw (adapt_1.south) -- (adapt_2.north);
\draw (adapt_2.south) -- (adapt_k.north);
\end{tikzpicture}}
\vspace*{-0.4cm}
\caption{\small Adapters used to synchronize $k>2$ modalities with sigmoid loss. }
\label{fig:multiplemodalitiesadapter}
\vspace*{-0.7cm}
\end{wrapfigure}
We can enforce this structure by adding an adapter which appends a modality-dependent suffix to the representation. Again the adapters can be implicitly captured by the relative bias since $\langle w_i, w_j\rangle = -\frac{1}{k-1}$ for all $i\neq j.$ 

\begin{observation} 
\label{obs:multipleapadapterimplicit}
For any $\{(U^{(1)}_i,U^{(2)}_i,\ldots, U^{(k)}_i)\}_{i = 1}^N$ and $\delta, \relativebias,t,$
\begin{align*}
& \losssigliprb\Big(\{
\adapter^\delta_{1}(U^{(1)}_i),
\ldots, 
\adapter^\delta_{k}(U^{(k)}_i)\}_{i = 1}^N;t,\relativebias\Big)\\
&\qquad\qquad= 
\losssigliprb\Big(\{U^{(1)}_i, \ldots, U^{(k)}_i)\}_{i = 1}^N;t\delta^2,\relativebias + \frac{(1-\delta)^2}{k-1}\Big).
\end{align*}
\end{observation}

\subsection{Ablation Studies}

\emph{Training the temperature and bias} is the key mechanism that drives the loss to zero for a wide range of configurations which we described as $(\margin,\relativebias)$-Constellations. We compared our proposal of training $\losssigliprb$ with trainable $\relativebias,t$ parameters against several alternatives. We concretely focus on the \emph{inner product separation condition} $\min_{i}\langle U_i,V_i\rangle\ge\max_{i\neq j}\langle U_i,V_j\rangle$ and corresponding margin. This is a key property of interest since, as explained, it determines the success of the model on retrieval via (approximate) nearest neighbor search. We also analyze the convergence of the loss to zero, which is an indicator for how many epochs of training a model needs till convergence.

\textbf{1. Training with fixed low inverse temperature ($t \lesssim 10$) and bias} as in the analysis of \cite{lee24analysissiglip} is a first natural alternative. We observed that in the contexts of synchronizing multiple embeddings (Figure \ref{fig:multiplemodalities}) and synchronizing with a locked encoder (Figure \ref{fig:lockedrepresenattion}), the resulting embeddings fail to satisfy the inner product separation condition or do so with a much smaller margin than models with trainable inverse temperature and (relative) bias.

\textbf{2.Training with fixed high inverse temperature ($t \gg 10$) and bias.} Any $(\margin, \relativebias)$-constellation is nearly a global minimum in the regime of large $t$, so one may expect similar performance to the trainable inverse temperature and bias model. This approach fails in practice since it does not allow the algorithm to gradually find the synchronized representations. The embeddings discovered are not useful towards retrieval as the inner-product separation fails and the loss does not approach zero (Figures \ref{fig:lockedrepresenattion} and \ref{fig:multiplemodalities}), even though representations with nearly-zero loss exist due to the low temperature.

\textbf{3. Training with bias parameterization.} Finally, we compared against the bias parameterization $\losssiglip$ used in \cite{zhai23siglip,tschannen2025siglip2}. While this model also generally led to $(\margin,\relativebias)$-Constellations which can be used for perfect retrieval, we observed \emph{slower convergence of the loss function} and \emph{smaller margin} due to a tendency towards zero relative bias (Figures \ref{fig:lockedrepresenattion} and \ref{fig:multiplemodalities}).

\section{Limitations and Future Directions}
\label{sec:limitations}
We provide the first theoretical analysis of synchronizing representations in the practically relevant regime $d\ll N \ll 2^d.$ A {theoretical} limitation is that while we identify global minimizers and empirically show that first-order methods such as Adam find them, we do not prove rigorous performance guarantees for first-order methods. Another theoretical limitation is that we do not fully resolve the combinatorial Problem \ref{problem:cardinalitybounds}, which as we point out is practically relevant for choosing the embedding dimension of encoders. Finally, we show that the parametrization of sigmoid loss with relative bias leads to more flexibility and faster convergence on synthetic data, but do not perform experiments with it on real data. We believe that all of these are exciting directions for future research.

We are not aware of any negative or direct impacts on society of our work.
The work can have indirect societal impact as the findings are relevant to modern large-scale machine learning systems.

\section*{Acknowledgments}

Research was supported, in part, by the United States Air
Force Research Laboratory and the United States Air Force
Artificial Intelligence Accelerator under Cooperative Agreement Number FA8750-19-2-1000. The views and conclusions contained in this document are those of the authors
and should not be interpreted as representing the official
policies, either expressed or implied, of the United States
Air Force or the U.S. Government. The U.S. Government
is authorized to reproduce and distribute reprints for Government purposes notwithstanding any copyright notation
herein.

\bibliography{ref}
\bibliographystyle{alpha}

\newpage

\appendix

\section{Omitted Proofs From \texorpdfstring{\cref{sec:geometryofzeroloss}}{Geometry Of ZeroLoss}}
\label{appendix:geometryofzeroloss}
\subsection{Global Minimizers of The Sigmoid Loss are \texorpdfstring{$(\margin,\relativebias)$-Constellations}{Constellations}}
We will repeatedly use the following fact which follows from $\log(1 + \exp(\kappa))\ge 0$ for any $\kappa\in\mathbb{R}.$

\begin{observation}
    \label{obs:boundbymaxandmin}
    For any $\{(U_i,V_i)\}_{i= 1}^N$ and $t,b,$ it holds that 
    \begin{align}
        & \max\Bigg(
        \max_i \log\Big(1 + \exp(-t \langle U_i, V_i\rangle + b)\Big), \max_{i\neq j} 
        \log\Big(1 + \exp(t \langle U_i, V_j\rangle - b)\Big)
        \Bigg) \label{align:lbineqonsiglip}\\
        &  \qquad\qquad\qquad
        \le
        \losssiglip(\{U_i\}_{i = 1}^N, \{V_i\}_{i=1}^N;t,b)\notag\\
        & \le  N^2\times 
        \max\Bigg(
        \max_i \log\Big(1 + \exp(-t \langle U_i, V_i\rangle + b)\Big), \max_{i\neq j} 
        \log\Big(1 + \exp(t \langle U_i, V_j\rangle - b)\Big)
        \Bigg)\label{align:ubineqonsiglip}.
    \end{align}
\end{observation}

\begin{proof}[Proof of \cref{thm:convergencetozeroloss}] Suppose that 
$\lim_{s\longrightarrow + \infty} \losssiglip(\{U^{(s)}_i\}_{i = 1}^N, \{V^{(s)}_i\}_{i=1}^N;t^{(s)},b^{(s)})= 0$ indeed holds. By, \eqref{align:lbineqonsiglip}, this means that
\begin{align*}
    & \lim_{s\longrightarrow + \infty}\log\Big(1 + \exp(-t^{(s)} \langle U_i^{(s)}, V_i^{(s)}\rangle + b^{(s)})\Big) = 0\; \qquad\forall i,\\
    & \lim_{s\longrightarrow + \infty}\log\Big(1 + \exp(t^{(s)} \langle U_i^{(s)}, V_j^{(s)}\rangle - b^{(s)})\Big) = 0\; \qquad\forall i\neq j.
\end{align*}
Equivalently, 
\begin{align*}
    & \lim_{s\longrightarrow + \infty} -t^{(s)} \langle U_i^{(s)}, V_i^{(s)}\rangle + b^{(s)} = -\infty\; \qquad\forall i,\\
    & \lim_{s\longrightarrow + \infty} t^{(s)} \langle U_i^{(s)}, V_j^{(s)}\rangle - b^{(s)}= -\infty\; \qquad\forall i\neq j.
\end{align*}
Equivalently, 
\begin{align*}
    & \lim_{s\longrightarrow + \infty} t^{(s)} \Big(\langle U_i^{(s)}, V_i^{(s)}\rangle - \frac{b^{(s)}}{t^{(s)}}\Big) = +\infty\; \qquad\forall i,\\
    & \lim_{s\longrightarrow + \infty} t^{(s)} \Big(\langle U_i^{(s)}, V_j^{(s)}\rangle - \frac{b^{(s)}}{t^{(s)}}\Big)= -\infty\; \qquad\forall i\neq j.
\end{align*}
In particular, as $t^{(s)}>0$ always, this means that for all large enough $s,$ the quantity $\relativebias^{(s)}\coloneqq \frac{b^{(s)}}{t^{(s)}}$ satisfies that 
\begin{align}    
\label{eq:upperboundrlinconvergence}
\langle U_i^{(s)}, V_i^{(s)}\rangle - \relativebias^{(s)}&\ge 0 \qquad\forall i,   \\
\label{eq:lowerboundrlinconvergence}
\langle U_i^{(s)}, V_j^{(s)}\rangle - \relativebias^{(s)}&\le 0 \qquad\forall i\neq j\,.
\end{align} 
However, as all $U_i^{(s)}, V_i^{(s)}$ are unit vectors, 
$\langle U_i^{(s)}, V_i^{(s)}\rangle\le 1, \langle U_i^{(s)}, V_j^{(s)}\rangle\ge -1$ holds for any $i,j,s.$ Hence, for all large enough $s,$ \eqref{eq:upperboundrlinconvergence} and \eqref{eq:lowerboundrlinconvergence} imply that 
$$
-1\le \relativebias^{(s)}\le 1.
$$
Now, observe that 
$\{U_1^{(s)}, U_2^{(s)}, \ldots, U_N^{(s)}, 
V_1^{(s)}, V_2^{(s)}, \ldots, V_N^{(s)}, \relativebias^{(s)}
\}\in (\mathbb{S}^{d-1})^{\otimes 2N}\times[-1,1].$ As $(\mathbb{S}^{d-1})^{\otimes 2N}\times[-1,1]$ is a compact set, 
$\{U_1^{(s)}, U_2^{(s)}, \ldots, U_N^{(s)}, 
V_1^{(s)}, V_2^{(s)}, \ldots, V_N^{(s)}, \relativebias^{(s)}
\}_{s= 1}^{+\infty}$ has a convergent subsequence. Suppose that it converges to $\{U_1, U_2, \ldots, U_N, 
V_1, V_2, \ldots, V_N, \relativebias\}.$ By \eqref{eq:upperboundrlinconvergence} and \eqref{eq:lowerboundrlinconvergence}, we have that 
\begin{align}
    & \langle U_i, V_i\rangle - \relativebias\ge 0 \qquad\forall i,\\
    & \langle U_i, V_j\rangle - \relativebias\le 0 \qquad\forall i\neq j.
\end{align}
Setting 
$$\margin = 
\min\Bigg(
\min_i \langle U_i, V_i\rangle - \relativebias,
\min_{i\neq j }\relativebias - 
\langle U_i, V_j\rangle
\Bigg)$$
gives the desired result.
\end{proof}

\begin{proof}[Proof of \cref{thm:ipimplieszeroloss}] Suppose that $\{(U_i,V_i)\}_{i=1}^N$ are a $(\margin,\relativebias)$-Constellation with some $\margin >0.$ Then, for any $t,$ it holds that 
\begin{align}
    & t(\langle U_i, V_i\rangle - \relativebias)\ge \margin\times t \qquad\forall i,\\
    & t(\langle U_i, V_j\rangle - \relativebias)\le -  \margin\times t\qquad\forall i\neq j.
\end{align}
In particular, by Observation \eqref{align:ubineqonsiglip}, it follows that 
\begin{align*}
& \losssiglip(\{U_i\}_{i = 1}^N, \{V_i\}_{i=1}^N;t,t\relativebias)\\
& \le N^2\times \log\Big(1 + \exp(- \margin\times t)\Big)\\
& \le N^2 \exp(- \margin\times t) = 
\exp(- \margin\times t + 2\log N/t) = 
\exp(- \margin\times t + o(t))
\end{align*}
which proves the convergence to zero. Choosing $(\margin^*, \relativebias^*)$  in this argument, where
\begin{align*}
    &\margin^* \coloneqq \frac{1}{2}(\min_{i}\langle U_i, V_i\rangle - \max_{i\neq j}\langle U_i, V_j\rangle)>\frac{1}{2}((\margin +\relativebias)- (-\margin+\relativebias)) = \margin>0,\\
    &\relativebias^* \coloneqq \frac{1}{2}(\min_{i}\langle U_i, V_i\rangle + \max_{i\neq j}\langle U_i, V_j\rangle),
\end{align*}
also proves that $
\inf_{b}
\losssiglip(\{U_i\}^N_{i=1}, \{V_i\}_{i=1}^N;t, b)\le  e^{- t\margin^* + o(t)}.
$

All that is left to show is that
$$
\inf_{b}
\losssiglip(\{U_i\}^N_{i=1}, \{V_i\}_{i=1}^N;t, b)\ge  e^{- t\margin^* + o(t)}.
$$
Equivalently, we can show that $
\inf_{b}
\losssiglip(\{U_i\}^N_{i=1}, \{V_i\}_{i=1}^N;t, b)\ge \log\Big(1 +  e^{- t\margin^* + o(t)}\Big)
$ since $\log(1 + \gamma) = \gamma + o(\gamma)$ as $\gamma\longrightarrow 0.$ However, by \eqref{align:lbineqonsiglip}, for the last inequality, it is enough to show that 
$$
\max\Bigg(
        \max_i \Big(-t \langle U_i, V_i\rangle + b)\Big), \max_{i\neq j} 
        \Big(t \langle U_i, V_j\rangle - b\Big)
        \Bigg)\ge - t\times\margin^*\,.
$$
Equivalently 
$$
\min\Bigg(
        \min_i \Big(\langle U_i, V_i\rangle - \frac{b}{t})\Big),\min_{i\neq j} 
        \Big( - t \langle U_i, V_j\rangle + \frac{b}{t}\Big)
        \Bigg)\le \margin^*\,.
$$
Suppose, for the sake of contradiction, that for some $b, t,$ we have that 
$$
\min\Bigg(
        \min_i \Big(\langle U_i, V_i\rangle - \frac{b}{t})\Big),\min_{i\neq j} 
        \Big( - t \langle U_i, V_j\rangle + \frac{b}{t}\Big)
        \Bigg) = \margin'>\margin^*.
$$
Then, 
$$
\min_{i}\langle U_i, V_i\rangle - \max_{i\neq j}\langle U_i, V_j\rangle\ge 
\frac{b}{t} + \margin' - 
(\frac{b}{t} - \margin') = 2\margin' >2\margin^*,
$$
which is a contradiction with the definition of $\margin^*.$
\end{proof}

\subsection{Robustness of Nearest-Neighbor Retrieval: Proof of \texorpdfstring{\cref{prop:zeroloss}}{RobustThm}}
\label{appendix:robustretrieval}
Suppose that 
\begin{align*}
& \expect_{(a_j)_{j = 1}^B\sim [N]^{\times k}}\Big[ \sum_{j = 1}^B \log \Big(1 + \exp(- t \langle U_{a_j}, V_{a_j}\rangle + b)\Big)\\
& \qquad\qquad\qquad\qquad\qquad\qquad+ 
    \sum_{i \neq j} \log \Big(1 + \exp(t \langle U_{a_i}, V_{a_j}\rangle - b)\Big)\Big]\le \xi \log 2.
\end{align*}
holds. Again,as the function $x\longrightarrow \log(1 + e^x)$ is non-negative, this implies that for some $0\le x \le \xi,$
\begin{align*}
& x \log 2 = 
\expect_{(a_j)_{j = 1}^B\sim [N]^{\times k}}\Big[ \sum_{j = 1}^B \log \Big(1 + \exp(- t \langle U_{a_j}, V_{a_j}\rangle + b)\Big)\Big]\\
& = 
B \times 
\expect_{j \sim \textsf{unif}([N])}\log \Big(1 + \exp(- t \langle U_{j}, V_{j}\rangle + b)\Big).
\end{align*}
However, whenever $t \langle U_{j}, V_{j}\rangle - b\le 0,$ then
$\log \Big(1 + \exp(- t \langle U_{j}, V_{j}\rangle + b)\Big)\ge \log 2.$ By Markov's inequality, it follows that 
$$
\prob_{j \sim \textsf{unif}([N])}[t \langle U_{j}, V_{j}\rangle - b\le 0] \le \frac{x }{B}.
$$
Hence, for all but at most a $\frac{x }{B}$fraction of the data indices $j,$ it follows that $\langle U_{j}, V_{j}\rangle > b/t.$

In the exact same way, for some $y\ge 0 $ such that $x + y\le \xi.$ it follows that 
$$
\prob_{i,j \sim [N]^{\times 2}}[t \langle U_{j}, V_{j}\rangle - b\ge 0] \le \frac{y }{B(B-1)}.
$$
Note that for every fixed $i,$ if $\prob_{j \sim \textsf{unif}([N])|_{j\neq i}}[t \langle U_{j}, V_{j}\rangle - b>0]>0,$ then 
$ \prob_{j \sim \textsf{unif}([N])|_{j\neq i}}[t \langle U_{j}, V_{j}\rangle - b>0]\ge 1/(N-1).$ Hence, one can similarly argue that for all but at most a
$\frac{y(N-1)}{B(B-1)}$ fraction of the data indices $i,$ it follows that $\langle U_{i}, V_{j}\rangle < b/t$ for all $j.$ Thus, for at least a  
$$1 - \frac{x }{B} - \frac{y(N-1)}{B(B-1)}$$ fraction of the indices $i,$ it follows that for any $j\neq i,$
$$
\langle U_{i}, V_{i}\rangle > b/t>
\langle U_{i}, V_{j}\rangle. 
$$
Clearly, for these indices, nearest neighbor search succeeds. Optimizing over $0\le x,0\le y, x+y\le \xi,$ we reach the conclusion.

\subsection{Global Minimizers of InfoNCE Loss}
\label{sec:infonce}
We claim that the following four results hold. Since the proofs are nearly identical to the ones for SigLIP, we only give one prove the third result.

\begin{theorem}
[Necessary Condition for InfoNCE Global Minimizers]
Suppose that any iterative algorithm produces a sequence $\{U^{(s)}_i\}_{i =1}^N, \{V^{(s)}_i\}_{i =1}^N,t^{(s)}>0$ for $s = 1,2, \ldots$ such that 
$$
\lim_{s\longrightarrow + \infty} \lossinfonce(\{U^{(s)}_i\}_{i = 1}^N, \{V^{(s)}_i\}_{i=1}^N;t^{(s)})= 0.
$$
Then, there exists some subsequence indexed by $(s_r)_{r=1}^{+\infty}$ such that
\begin{equation}
\begin{split}
    \lim_{r \longrightarrow +\infty} U_i^{(s_r)} = U_i,\quad  
    \lim_{r \longrightarrow +\infty} V_i^{(s_r)} = V_i \text{ for all }i
\end{split}
\end{equation}
and there exists some $\margin\ge 0$ and $\{\relativebias(i)\}_{i = 1}^N$ such that $\{(U_i,V_i)\}_{i= 1}^N,\margin,\relativebias$ satisfy:
\begin{equation}
\label{eq:fullinfonceglobmin}
    \begin{split}
        &\langle U_i, V_i\rangle \ge \relativebias(i) + m\qquad \forall i,\\
        &\langle U_i, V_j\rangle \le \min(\relativebias(i), \relativebias(j)) - m\qquad \forall i\neq j.\\
    \end{split}
\end{equation}
\end{theorem}
\begin{theorem}[Sufficient Condition for InfoNCE Global Minimizers]
Suppose that $\{(U_i,V_i)\}_{i=1}^N \in \mathbb{S}^{d-1}$ satisfies \eqref{eq:fullinfonceglobmin} for some $\margin>0$ and $\{\relativebias(i)\}_{i = 1}^N.$ Then,
$$
\lim_{t\longrightarrow  + \infty}
\lossinfonce(\{U_i\}_{i=1}^N,\{V_i\}_{i=1}^N;t) = 0\,.
$$    
Moreover, for
$\margin^* \coloneqq \frac{1}{2}(\min_{i}\langle U_i, V_i\rangle - \max_{i\neq j}\langle U_i, V_j\rangle),$  
$$
\inf_{b}
\lossinfonce(\{U_i\}^N_{i=1}, \{V_i\}_{i=1}^N;t)= e^{- t\margin^* + o(t)}
.$$
\end{theorem}

For the other two results, consider the one-sided InfoNCE defined by 
$$\lossinfonceu(\{(U_i,V_i)\}_{i = 1}^N;t) = 
    - \frac{1}{N}\sum_{i = 1}^N \log \frac{\exp(t \langle U_i, V_i\rangle )}{\sum_{j}\exp(t \langle U_i, V_j\rangle)}
   .
$$
\begin{theorem}
[Necessary Condition for one-sided InfoNCE Global Minimizers]
\label{thm:one-sidedinfoncenecessary}
Suppose that any iterative algorithm produces a sequence $\{U^{(s)}_i\}_{i =1}^N, \{V^{(s)}_i\}_{i =1}^N,t^{(s)}>0$ for $s = 1,2, \ldots$ such that 
$$
\lim_{s\longrightarrow + \infty} \lossinfonceu(\{U^{(s)}_i\}_{i = 1}^N, \{V^{(s)}_i\}_{i=1}^N;t^{(s)})= 0.
$$
Then, there exists some subsequence indexed by $(s_r)_{r=1}^{+\infty}$ such that 
\begin{equation}
\begin{split}
    \lim_{r \longrightarrow +\infty} U_i^{(s_r)} = U_i,\quad  
    \lim_{r \longrightarrow +\infty} V_i^{(s_r)} = V_i \text{ for all }i
\end{split}
\end{equation}
and there exists some $\margin\ge 0$ and $\{\relativebias(i)\}_{i = 1}^N$ such that $\{(U_i,V_i)\}_{i= 1}^N,\margin,\relativebias$ satisfy \eqref{eq:InfoNCEGlobMinima}:
\begin{equation}
    \begin{split}
        &\langle U_i, V_i\rangle \ge \relativebias(i) + m\qquad \forall i,\\
        &\langle U_i, V_j\rangle \le \relativebias(i) - m\qquad \forall i\neq j.\\
    \end{split}
\end{equation}
\end{theorem}
\begin{theorem}[Sufficient Condition for one-sided InfoNCE Global Minimizers]
Suppose that $\{(U_i,V_i)\}_{i=1}^N \in \mathbb{S}^{d-1}$ satisfies \eqref{eq:InfoNCEGlobMinima} for some $\margin>0$ and $\{\relativebias(i)\}_{i = 1}^N.$ Then,
$$
\lim_{t\longrightarrow  + \infty}
\lossinfonceu(\{U_i\}_{i=1}^N,\{V_i\}_{i=1}^N;t) = 0\,.
$$    
Moreover, for
$\margin^* \coloneqq \frac{1}{2}(\min_{i}\langle U_i, V_i\rangle - \max_{i\neq j}\langle U_i, V_j\rangle),$  
$$
\inf_{b}
\lossinfonceu(\{U_i\}^N_{i=1}, \{V_i\}_{i=1}^N;t)= e^{- t\margin^* + o(t)}
.$$
\end{theorem}

\begin{proof}[Proof of \cref{thm:one-sidedinfoncenecessary}]
The key step in the proof is that the loss can be equivalently re-written as 
\begin{equation*}
    \begin{split}
        \lossinfonceu(\{(U_i,V_i)\}_{i = 1}^N;t) 
        & = 
    - \frac{1}{N}\sum_{i = 1}^N \log \frac{\exp(t \langle U_i, V_i\rangle )}{\sum_{j}\exp(t \langle U_i, V_j\rangle)}\\
      & = \frac{1}{N}\sum_{i = 1}^N \log \Big(1 + \sum_{j\neq i}\exp\big(t( \langle U_i, V_j\rangle - 
      \langle U_i, V_i\rangle
      \big)\Big).
    \end{split}
\end{equation*}
Hence, if
$\lim_{s\longrightarrow + \infty} \lossinfonceu(\{U^{(s)}_i\}_{i = 1}^N, \{V^{(s)}_i\}_{i=1}^N;t^{(s)})= 0$ indeed holds, then
\begin{align*}
    & \lim_{s\longrightarrow + \infty}\log\Big(1 + \sum_{j\neq i}\exp\big(t^{(s)} (\langle U_i^{(s)}, V_j^{(s)}\rangle - \langle U_i^{(s)}, V_i^{(s)}\rangle\big)\Big) = 0\; \qquad\forall i.
\end{align*}
Equivalently, 
\begin{align}
\label{eq:infoncepartconv}
    & \lim_{s\longrightarrow + \infty} t^{(s)} \big(\langle U_i^{(s)}, V_j^{(s)}\rangle -
    \langle U_i^{(s)}, V_i^{(s)}\rangle\big)
    = -\infty\; \qquad\forall i\neq j.
\end{align}
As all $U_i^{(s)}, V_i^{(s)}$ are unit vectors, 
$\langle U_i^{(s)}, V_i^{(s)}\rangle\le 1, \langle U_i^{(s)}, V_j^{(s)}\rangle\ge -1$ holds for any $i,j,s.$ Hence, for all large enough $s,$ 
$$
-2\le \langle U_i^{(s)}, V_j^{(s)}\rangle -
    \langle U_i^{(s)}, V_i^{(s)}\rangle\le 2.
$$
In particular, since $t^{(s)}>0,$ \cref{eq:infoncepartconv} implies that $t^{(s)}\longrightarrow+\infty$ as $s\longrightarrow +\infty$ and  
\begin{align}
\label{eq:conditionnegativeinfonce}
    \langle U_i^{(s)}, V_j^{(s)}\rangle -
    \langle U_i^{(s)}, V_i^{(s)}\rangle<0 \text{ for all $i,j$ and large enough } s.
\end{align}
Now, observe that 
$\{U_1^{(s)}, U_2^{(s)}, \ldots, U_N^{(s)}, 
V_1^{(s)}, V_2^{(s)}, \ldots, V_N^{(s)}
\}\in (\mathbb{S}^{d-1})^{\otimes 2N}.$ As $(\mathbb{S}^{d-1})^{\otimes 2N}$ is a compact set, 
$\{U_1^{(s)}, U_2^{(s)}, \ldots, U_N^{(s)}, 
V_1^{(s)}, V_2^{(s)}, \ldots, V_N^{(s)}
\}_{s= 1}^{+\infty}$ has a convergent subsequence. Suppose that it converges to $\{U_1, U_2, \ldots, U_N, 
V_1, V_2, \ldots\}.$ By \eqref{eq:conditionnegativeinfonce} and setting 
$$\margin = 
\min\big(
\langle U_i, V_i\rangle -  
\langle U_i, V_j\rangle
\big)/2,$$
we conclude.
\end{proof}
\subsection{Triplet Loss}
\label{sec:relationtotripletloss}
In the context of synchronizing embeddings, the triplet loss function \cite{Schroff2015tripletloss} with hyperparameter margin $\alpha$ takes form 
\begin{equation}
    \label{eq:deftriplet}
    \begin{split}
    & \losstriplet(\{(U_i,V_i)\}_{i = 1}^N;\alpha) = 
    \sum_{i \neq j} \max(\|U_i - V_i\|_2^2 -\|U_i - V_j\|_2^2 + \alpha,0).
    \end{split}
\end{equation}

\begin{observation}
\label{obs:tripletlossproof}
Suppose that $\{(U_i, V_i)\}_{i = 1}^N$ is a  $(\margin,\relativebias)$-Constellation. Then, 
for any $\alpha \le 4\margin,$ it is also the case that
$$
\losstriplet(\{(U_i,V_i)\}_{i = 1}^N;\alpha) = 0.
$$
\end{observation}
\begin{proof}
Suppose that $\{(U_i, V_i)\}_{i = 1}^N$ is a  $(\margin,\relativebias)$-Constellation. Then, for any $i,\neq j,$ we have that 
\begin{align*}
    & \|U_i - V_i\|_2^2 -\|U_i - V_j\|_2^2\\
    & = 2 - 2\langle U_i,V_i\rangle - (2 - 2\langle U_i,V_j\rangle)\\
    & = 2(\langle U_i,V_j\rangle - \langle U_i,V_i\rangle)\\
    & \le 2(- \margin + \relativebias - \margin -\relativebias) = -4\margin. 
\end{align*}
This finishes the proof.
\end{proof}

\begin{observation} Suppose that $\{(U_i, V_i)\}_{i = 1}^N$ is a global minimizer of the one-sided InfoNCE loss with margin $m$ as in \cref{eq:InfoNCEGlobMinima} or the regular InfoNCE as in \cref{eq:fullinfonceglobmin}. Then, also for any $\alpha \le 2m,$
    $$
    \losstriplet(\{(U_i,V_i)\}_{i = 1}^N;\alpha) = 0.
    $$
\end{observation}
\begin{proof}
Note that for any $i\neq j,$
\begin{align*}
& \|U_i - V_i\|_2^2 -\|U_i - V_j\|_2^2 + \alpha = 
\langle U_i, V_j\rangle - 
\langle U_i, V_i\rangle +\alpha \\
& \le 
(\relativebias(i) - m) -
(\relativebias(i) + m) + \alpha = 
\alpha - 2m \le 0.
\end{align*}
\end{proof}
\begin{observation} Suppose that $\{(U_i, V_i)\}_{i = 1}^N$ is a global minimizer of the triplet loss for some $\alpha>0.$ Then, also for any $m \le \alpha/2$ and appropriate $\{\relativebias(i)\}_{i = }^N,$ 
$\{(U_i, V_i)\}_{i = 1}^N$ satisfy \cref{eq:InfoNCEGlobMinima} and hence
    $$
    \lim_{t\longrightarrow + \infty}\lossinfonceu(\{(U_i, V_i)\}_{i = 1}^N;t) =  0.
    $$
\end{observation}
\begin{proof}
Note that 
$$
\|U_i - V_i\|_2^2 -\|U_i - V_j\|_2^2 + \alpha = 
2(\langle U_i, V_j\rangle - \langle U_i, V_i\rangle +\alpha/2).
$$
Hence, if $\losstriplet = 0,$ for any $i\neq j,$
$$
\langle U_i, V_j\rangle - \langle U_i, V_i\rangle +\alpha/2\le 0.
$$
Defining 
\begin{equation*}
    \begin{split}
        & \relativebias(i) = \min_j\big(\langle U_i, V_i\rangle - \langle U_i, V_j\rangle\big),\\ 
        & m = \alpha/2,
    \end{split}
\end{equation*}
we reach the conclusion.
\end{proof}

Note that while global minimizers of InfoNCE and triplet loss coincide, the global minimizers of sigmoid loss are only a subset of these configurations. This is clearly illustrated by 
\cref{fig:infoncevssiglip}.
\section{Proof of \texorpdfstring{\cref{thm:upperboundconst}}{Theorem}: Dimension vs Size tradeoff}
\label{sec:upperboundconstellationsize}
\begin{proof}[Proof of \cref{thm:upperboundconst}]
Let 
\[
H\sim\mathrm{Unif}(S^{d-1}), 
\qquad 
C(H)=\{\,i:\langle c_i,H\rangle>\delta\}, 
\quad
N'=|C(H)|,
\]
where $c_i=(U_i+V_i)/2$ and $\delta\in(0,1)$ will be chosen later.  
\[
\|c_i\|^2=\frac{1+\langle U_i,V_i\rangle}{2}\in
\Bigl[\tfrac{1+\margin+\relativebias}{2},1\Bigr], \text{ so } \|c_i\|\ge\sqrt{(1+\margin+\relativebias)/2}
\] Then the inner product satisfies
\[
\Pr\bigl[\langle c_i,H\rangle>\delta\bigr]
=\Gamma_d\!\Bigl(\tfrac{\delta}{\|c_i\|}\Bigr)
\;\ge\;
\Gamma_d\!\Bigl(\delta\sqrt{\tfrac{2}{1+\margin+\relativebias}}\Bigr),
\]
where $\Gamma_d(x)=\Pr[H_1>x]$ is strictly decreasing in $x$.  By linearity of expectation,
\[
\mathbb{E}[N']\;\ge\;N\,\Gamma_d\!\Bigl(\delta\sqrt{\tfrac{2}{1+\margin+\relativebias}}\Bigr),
\]
so there exists a realization of $H$ with
\begin{equation}\label{eq:ENprime}
N'\;\ge\;N\,\Gamma_d\!\Bigl(\delta\sqrt{\tfrac{2}{1+\margin+\relativebias}}\Bigr).
\end{equation}

Define for the index set $C=C(H)$ the sums
\[
U_C=\sum_{i\in C}U_i,\quad V_C=\sum_{i\in C}V_i,\quad
x_i=U_i-V_i,\quad x_C=\sum_{i\in C}x_i,
\]
and set
\[
A_C=\sum_{i\in C}\langle U_i,V_i\rangle,\quad
B_C=\sum_{\substack{i,j\in C\\i\neq j}}\langle U_i,V_j\rangle,
\quad
\xi_C=\frac{1}{N'}\Bigl(\sum_{i\in C}\|x_i\|^2\Bigr)-\|\frac{1}{N'}x_C\|^2\ge0.
\]
A direct expansion shows
\begin{equation}\label{eq:master-C}
\|U_C+V_C\|^2
=
N'^2(2-\xi_C)\;-\;2(N'-2)\,A_C\;+\;4\,B_C.
\end{equation}
On the other hand,
\[
\Bigl\|\sum_{i\in C}c_i\Bigr\|
  \;=\;\frac12\|U_C+V_C\|
  \;\ge\;\sum_{i\in C}\langle c_i,H\rangle
  \;>\;N'\,\delta,
\]
so
\begin{equation}\label{eq:energy-C}
\|U_C+V_C\|^2>4N'^2\delta^2.
\end{equation}
Combining \eqref{eq:master-C} and \eqref{eq:energy-C}, and using
$A_C\ge (\margin+\relativebias)N'$ and $B_C\le (\relativebias-\margin)N'(N'-1)$, yields
\[
4N'^2\delta^2
\;\le\;
N'^2(2-\xi_C)\;-2(N'-2)\,(\margin+\relativebias)N'\;+4\bigl[(\relativebias-\margin)\,N'(N'-1)\bigr]\]
Which means \(
N'\bigl(3\margin-1+2\delta^2 - \relativebias\bigr)\;\le\;4\margin.
\) Where in the last reduction we dropped the \(\xi_C \ge 0\) term. Hence, whenever \(2\delta^2>1-3\margin+\relativebias\),
\begin{equation}\label{eq:bound-C}
N'\;\le\;\frac{4\margin}{2\delta^2-(1-3\margin + \relativebias)}.
\end{equation}

Combine \eqref{eq:ENprime} and \eqref{eq:bound-C} to get
\[
N\;\le\;
\frac{4\margin}{2\delta^2-(1-3\margin+\relativebias)}\;\Gamma_d\!\Bigl(\delta\sqrt{\tfrac{2}{1+\margin+\relativebias}}\Bigr)^{-1}.
\]
Recalling Shannon's asymptotic lower-bound
\(\Gamma_d(\cos\theta)=\exp\{d\log\sin\theta+o_\theta(d)\}\) \cite[(11)]{shannon59lb}
with \(\cos\theta=\delta\sqrt{2/(1+m+\relativebias)}\) and correponding $\sin \theta = \sqrt{1-\cos^2\theta} = 
\sqrt{1 - \frac{1-3\margin + \relativebias}{1 + \margin + \relativebias}},
$ the bound is optimized by choosing 
\(\delta =\sqrt{\tfrac{1-3m+\relativebias}{2}}\). This results in the claimed bound
\[\;N\le\exp(-d\log\sin\theta+o_\theta(d)).\qedhere\]
\end{proof}
\section{Omitted Proofs from \texorpdfstring{\cref{sec:modalitygap}}{Modality Gap}: Combinatorics of The Modality Gap}
\label{appendix:modalitygapproofs}

Here, we analyze configurations $\{(U_i, V_i)\}_{i = 1}^N\in (\mathbb{S}^{d-1}\times\mathbb{S}^{d-1})^{\times N}$ with the following property:
\begin{equation}
\label{eq:conditionsofmodality}
    \begin{split}
        & \langle U_i, V_i\rangle >0 \; \forall i,\\
        & \langle U_i, V_j\rangle <0 \; \forall i\neq j.\\
    \end{split}
\end{equation}
Ultimately, we aim to prove \cref{thm:sperablemodalities}. We also prove several other facts on the way. Our proofs are based on simple facts from convex geometry which we introduce now. One can find more, for example, in the excellent book \cite{BoltyanskiVladimir2012EiCG}.

\subsection{Preliminaries from Convex Geometry}
A set $K \subseteq\mathbb{R}^d$ is called \emph{convex} if for any $\alpha\in [0,1]$ and any $p,q\in K,$ it is also the case that $\alpha p + (1-\alpha)q\in K.$ In particular, for any points $p_1, p_2,\ldots, p_k\in \mathbb{R}^d,$ the following two sets are convex. The \emph{convex hull} defined by
$$
\conv(p_1, p_2, \ldots, p_n)\coloneqq 
\Big\{\alpha_1p_1,\alpha_2p_2 + \cdots + \alpha_kp_k\; : \; 
\alpha_i \ge 0\qquad \forall i\text{ and }
\sum_{i= 1}^n\alpha_i = 1\Big\}
$$
and the \emph{convex cone} defined by 
\begin{align*}
& \cone(p_1, p_2, \ldots, p_n)\coloneqq 
\Big\{\alpha_1p_1,\alpha_2p_2 + \cdots + \alpha_kp_k\; : \; 
\alpha_i \ge 0 \qquad\forall i\text{ and }
\sum_{i= 1}^n\alpha_i \le 1\Big\}\\
& = 
\conv(p_1, p_2, \ldots, p_n,0).
\end{align*}
We also introduce the dual cone. For a set $S\subseteq \mathbb{R}^n,$ the \emph{dual cone} is given by 
$$
\dualcone(S)\coloneqq \{v \in \mathbb{R}^n\; : \; \langle v, x\rangle \ge 0 \quad \forall x\in S\}.
$$
We will use the following classic theorems from convex geometry.

\begin{theorem}[Helly {\cite{Helly1923}}]
\label{thm:helly}
Let \( X_1, X_2, \dots, X_n \) be a finite collection of convex sets in \( \mathbb{R}^d \). If the intersection of every \( d+1 \) of these sets is nonempty, then the intersection of all the sets is nonempty. Formally,
\[
\text{If } \bigcap_{i \in I} X_i \neq \emptyset \quad \text{for all } I \subset \{1, 2, \dots, n\} \text{ with } |I| = d+1,
\text{then } \bigcap_{i=1}^{n} X_i \neq \emptyset.
\]
\end{theorem}

\begin{theorem}[Carathéodory {\cite{Carathodory1911berDV,Steinitz1913}}]
\label{thm:caratheodoery}

Let \( A \subseteq \mathbb{R}^d \). If \( \mathbf{x} \in \text{conv}(A) \), then there exists a set \( B \subseteq A \) such that \( |B| \leq d+1 \) and \( \mathbf{x} \in \conv(B) \).

\end{theorem}

\begin{theorem}[Hyperplane Separation Theorem {\cite{minkowski1910}}]
\label{thm:hyperplaneseparation}
Let \( X \) and \( Y \) be two nonempty, disjoint convex sets in \( \mathbb{R}^d \). Then there exists a nonzero vector \( {a} \in \mathbb{R}^d \) and a scalar \( b \) such that
\begin{align*}
\langle {a} , {x} \rangle \leq b \quad \text{for all } {x} \in X,\\
\langle {a} , {y}\rangle  \geq b \quad \text{for all } {y} \in Y.
\end{align*}
\end{theorem}
\subsection{Combinatorics of Modality Gap}
\begin{proposition}\label{hpositive}
If \eqref{eq:conditionsofmodality} hold and $N\ge d+2,$ then there exists some $h \in \mathbb{S}^{d-1}$ such that $\langle h, U_i\rangle >0$ for all $i.$
\end{proposition}
\begin{proof}
    For each $i=1,\dots,N$, define the open half‐space
\[
  H_i = \{\,x\in\mathbb{R}^d:\ \langle U_i, x\rangle > 0\}.
\]
Each $H_i$ is convex. We first show that any subcollection of $d+1$ of these
half‐spaces has nonempty intersection.  Indeed, pick distinct indices
$i_1,\dots,i_{d+1}$; since $N\ge d+2$, there is an index
$j\notin\{i_1,\dots,i_{d+1}\}$.  By
\eqref{eq:conditionsofmodality}, for each $k=1,\dots,d+1$,
\[
  \langle U_{i_k}, V_j\rangle < 0 \text{ , so }
  \bigl\langle U_{i_k}, -V_j\bigr\rangle > 0,
\]
and $-V_j\in\bigcap_{k=1}^{d+1}H_{i_k}\neq\varnothing$.

Since every $d+1$ of the $H_i$ intersect and $N\ge d+2$, Helly’s theorem implies
\[
  \bigcap_{i=1}^N H_i \;\neq\;\varnothing.
\]
Choose any
$h_0\in\bigcap_{i=1}^N H_i$.  Then $\langle h_0,U_i\rangle>0$ for all $i$.
Setting $h = h_0/\|h_0\|\in\mathbb{S}$ preserves these strict
inequalities.
\end{proof}
\begin{proposition} \label{prop:convex-version} If \eqref{eq:conditionsofmodality} hold and $N\ge d+2,$ then there exists some $h \in \mathbb{R}^{d}$ such that $\langle h, U_i\rangle >0$ for all $i$ and 
$h \in \conv(U_1, U_2, \ldots, U_N).$
\end{proposition}
\begin{proof}
Let $C:=\operatorname{conv}(U_1,\dots ,U_N)\subset\mathbb R^{d}$, a compact
convex set.  
Take the unit vector $h$ given by \cref{hpositive} and denote by  \(h'\)
its projection onto $C$. This implies the fact that that the hyperplane
\[
  H := \bigl\{x\in\mathbb R^d : 
           \langle h-h',\,x-h'\rangle = 0\bigr\}
\]
\emph{supports} the set $C$ at the point $h'$:  
all points of $C$ (hence each $U_i$) lie in the closed half-space
\begin{equation}\label{eq:halfspace}
  H^- := \bigl\{x\in\mathbb R^d : 
             \langle h-h',\,x-h'\rangle \le 0\bigr\},
\end{equation}
whereas $h$ itself belongs to the opposite open half-space  
$H^+ := \{x : \langle h-h',x-h'\rangle>0\}$.
Choose an orthonormal basis $\{e_1,\dots ,e_d\}$ with
\[
  e_1 := \frac{h-h'}{\|h-h'\|}.
\]
In these coordinates
\[
  h = h' + \alpha e_1, 
  \quad\alpha:=\|h-h'\|>0,
  \qquad
  h' = 0 \cdot e_1+ h'_\perp,
\]
while every $U_i$ has a decomposition
\(
   U_i = U_{i,1}\,e_1 + U_{i,\perp}
\)
with \(U_{i,1}\le 0\).

For each $i$,
\[
  \langle h',U_i\rangle - \langle h,U_i\rangle
   = \langle h'-h,\,U_i\rangle
   = -\alpha\,U_{i,1}
   \;\ge\;0.
\]
Because $\langle h,U_i\rangle>0$,
we conclude that
\begin{align}\label{eq:hprime-positive}
  \langle h',U_i\rangle > 0 \qquad(i=1,\dots ,N).
\end{align}    
\end{proof}

\begin{proof}[Proof of \cref{thm:sperablemodalities}]
Let $h$ be the vector provided by \cref{prop:convex-version},
so $h\in\operatorname{conv}(U_1,\dots ,U_N)$ and
$\langle h,U_i\rangle>0$ for every $i$.  
Set the cone
\[
  C' := \operatorname{cone}\{U_1,\dots ,U_N\}
      =\Bigl\{\sum_{i=1}^N a_iU_i \,:\, a_i\ge 0\Bigr\}.
\]
Because each $U_i$ has positive dot product with $h$,
define the affine hyperplane
\[
  H := \bigl\{x\in\mathbb R^{d}: \langle x,h\rangle = 1\bigr\}.
\]
Every ray $\{\lambda U_i:\lambda>0\}$ meets $H$ once, namely at
\[
  Q_i := \frac{1}{\langle U_i,h\rangle}\,U_i\;\in H.
\]
Consequently
\begin{equation}\label{eq:CcapH}
  C'\cap H \;=\; \operatorname{conv}\{Q_1,\dots ,Q_N\}.
\end{equation}
The the vector $h^\dagger$ in $H$ which is parallel to $h$. $\langle h^\dagger, h \rangle = 1$, so $h^\dagger$ is $h$ rescaled by a positive scalar. Since $h^\dagger$ also lies in $\operatorname{conv}(U_1,\dots ,U_N)\subset C'$,
we have $h^\dagger\in C'\cap H$.
The hyperplane $H$ is $(d-1)$–dimensional.  By
Carathéodory’s theorem in $\mathbb R^{d-1}$,
there is a subset $S\subset\{1,\dots ,N\}$ with
\(
  |S|\le d
\)
and weights $\lambda_i\ge 0$, $\sum_{i\in S}\lambda_i=1$, such that
\begin{equation}\label{eq:Caratheodory}
  h^\dagger \;=\;
  \sum_{i\in S}\lambda_i\,Q_i
  \;=\;
  \sum_{i\in S}\lambda_i\,
  \frac{U_i}{\langle U_i,h\rangle}.
\end{equation}

Fix $k\notin S$.  Using \cref{thm:caratheodoery} and
\eqref{eq:conditionsofmodality},
\[
  \langle h^\dagger,V_k\rangle
    = \sum_{i\in S}\lambda_i\,
      \frac{\langle U_i,V_k\rangle}{\langle U_i,h\rangle}
    < 0,
\]
because each numerator $\langle U_i,V_k\rangle$ is negative and each
denominator $\langle U_i,h\rangle$ is positive.  Hence
$\langle h^\dagger,V_k\rangle<0$ for every $k\notin S$.
Only the (at most) $d$ indices in $S$ may give a non–negative value. Note that $h^\dagger$ is already on the unit circle.
\end{proof}
Now we prove that there is a construction for which this bound is almost tight and we can separate all but at least $d-1$ vectors with a hyperplane.
\begin{example}[Tightness of \cref{thm:sperablemodalities}]
\label{example:dminus3nonsep}
There exists a set of vectors $\{(U_i, V_i)\}_{i = 1}^N$ such that $\langle U_i, V_i\rangle>0$ for each $i,$ 
$\langle U_i, V_j\rangle<0$ for each $i\neq j,$ and for any $h\in \mathbb{S}^{d-1},$ at least for $d-1$ values of $i,$ it holds that 
$\langle h, U_i\rangle$ and $\langle h, V_i\rangle$ have the same sign.
\end{example}
\begin{proof}
    % \begin{figure}
    %     \centering
    %     \includegraphics[width=0.5\linewidth]{diagrams/octant.png}
    %     \caption{The region of points whose dot product with $U_1, U_2, U_3$ is negative in 3D.}
    %     \label{fig:octant}
    % \end{figure}

    For the construction for $d = 3$, note that we can take two parallel $k$-gons equally far from the equator of the sphere such that the zeniths of their points for one of them is $\pi/4+\delta$ and for the other one is $3/4\pi-\delta$ and by taking $\delta \to 0$ we can ensure that the dot product between corresponding pairs is positive and the dot product between non-matching pairs is negative. Now note that by taking $k$ sufficiently large and the $\delta$ sufficiently small. The intersection of that configuration with the wedge with dihedral angle $\alpha$ contains at least $N-2$ points from one of the $k$-gons (the $U$'s) and $N-2$ points from the other (the $V$'s) for any value of $\alpha$. See \cref{fig:wedge}. 
    \begin{figure}
        \centering
        \includegraphics[width=0.5\linewidth]{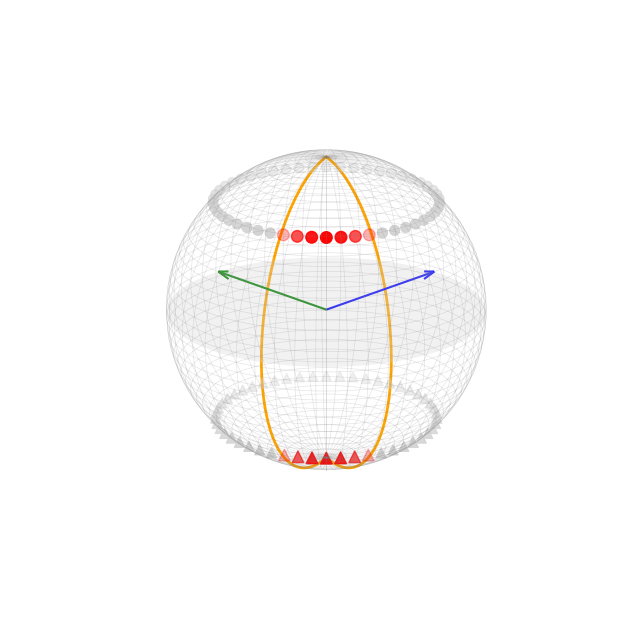}
        \caption{Construction in 3D.}
        \label{fig:wedge}
    \end{figure}
    
    Label those as 
\[
U_1,\dots,U_{N-2}
\quad\text{and}\quad
V_1,\dots,V_{N-2}.
\]
Finally place two more pairs on the equator:
\[
U_{N-1}=V_{N-1},\quad U_{N}=V_{N}
\]
chosen so that 
\(\langle U_{N-1},U_N\rangle<0\)
and both make negative dot‐products with each wedge boundary ray.
This completes the \(d=3\) example.

For the general case ($d>3$). Let \(\omega_1,\dots,\omega_{d+1}\) be the vertices of a regular simplex in \(\mathbb{R}^d\).
Set
\[
U_i = V_i = \omega_i,
\quad
i=1,\dots,d-3.
\]
These occupy a \((d-3)\)–dimensional subspace.  In the orthogonal complement
(which is 3–dimensional), embed the \(d=3\) construction,
obtaining pairs \((U_{d-2},V_{d-2}),\dots,(U_N,V_N)\).
Finally, pick a small \(\varepsilon>0\)
and renormalize:
\[
U_i'=\varepsilon\,\omega_{d-2}
   \;+\;\sqrt{1-\varepsilon^2}\;U_i,
\quad
V_i'=\varepsilon\,\omega_{d-2}
   \;+\;\sqrt{1-\varepsilon^2}\;V_i,
\quad
i=d-2,\dots,N.
\]
For \(\varepsilon\) sufficiently small, all the required dot‐product
signs are preserved, and since the configuration was orthogonal to $U_i$ for $i = 1, 2, \dots, d-3$, $$\langle U_i, V_j \rangle = -\epsilon \frac{1}{d} < 0, \forall i = 1, 2, \dots, d-3, j = d-2, d-1, \dots, N,$$ as needed.
\end{proof}

\begin{proposition} \label{prop:uconvex}If \eqref{eq:conditionsofmodality} hold then 
$U_i \not \in \cone (\{U_j\}_{j \neq i})$ for any $i.$
\end{proposition}
\begin{proof}
Suppose, to the contrary, that for some fixed \(i\) there exist scalars \(a_j\ge0\) for \(j\neq i\) such that
\[
U_i \;=\;\sum_{j\neq i} a_j\,U_j.
\]
Taking the inner product with \(V_i\) gives
\[
\langle V_i, U_i\rangle
\;=\;
\sum_{j\neq i} a_j\,\langle V_i, U_j\rangle.
\]
Since by \eqref{eq:conditionsofmodality} we have \(\langle V_i, U_j\rangle<0\) for all \(j\neq i\) and each \(a_j\ge0\), the right–hand side is nonpositive.  But the left–hand side is strictly positive, a contradiction.  Therefore \(U_i\notin\cone(\{U_j\}_{j\neq i})\).
\end{proof}
\begin{proposition} If $d = 2$ and $N\ge 4,$ there does not exist a configuration satisfying \eqref{eq:conditionsofmodality}. 
\end{proposition}
\begin{proof}
Because $N\ge d+2$, \cref{prop:convex-version} gives a unit vector $h$ with $\langle h,U_i\rangle>0$ for every $i$.  
Rotate so $h=(1,0)$; all $U_i$ now lie in the open right half-plane $x>0$.  
Write their polar angles in $(-\frac{\pi}{2},\frac{\pi}{2})$ as
\[
-\tfrac{\pi}{2}<\theta_1<\theta_2<\dots<\theta_N<\tfrac{\pi}{2}.
\]
Now note that $U_2 \in \cone (U_1, U_3)$ which is a contradiction by \cref{prop:uconvex}. Therefore configuration \eqref{eq:conditionsofmodality} cannot exist when $d=2$ and $N\ge4$.
\end{proof}
\section{Further Experiments and Experimental Details}
\label{sec:experiments}
For the experiments in \cref{appendix:experimentsimagenet}, we used a single A100 GPU. All other experiments are done on a standard CPU and take at most several minutes.
\subsection{Experiments on ImageNet}
\label{appendix:experimentsimagenet}

\begin{wrapfigure}{r}{.4\textwidth}
    \includegraphics[width = \linewidth]{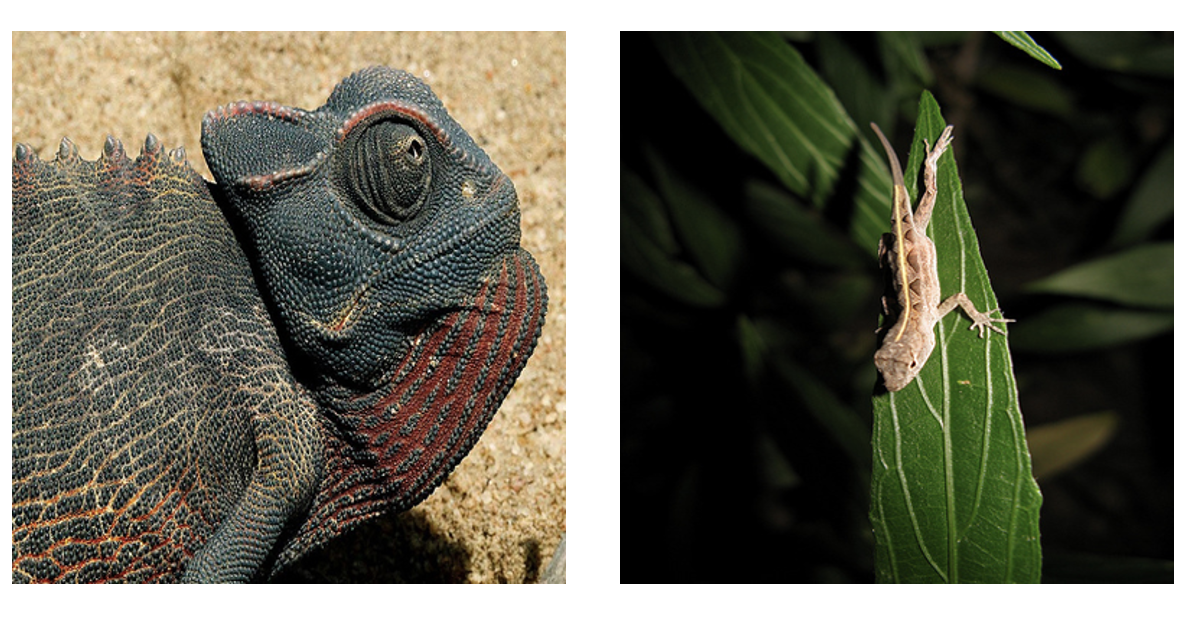}
    \caption{\small{``African chameleon'' on the right and ``American chameleon'' on the left from the ImageNet validation dataset. The B/16 model representation of the image of  ``American chameleon'' was closer to the representation of the word \textsf{African chameleon} than that of 
    \textsf{American chameleon}}}
    \label{fig:chameleons}
\end{wrapfigure}

In Figures \ref{fig:siglip_ipseparation} and \ref{fig:modalitygasiglip}, we performed experiments on real data with the SigLIP implementation. While a next generation vision-language encoder was introduced with the SigLIP 2 paper \cite{tschannen2025siglip}, we opted to use the original SigLIP model rather than SigLIP 2 because SigLIP 2's enhanced training recipe -- incorporating auxiliary decoder, self-distillation, and masked-prediction losses -- would confound our ability to isolate the impact of the core Sigmoid Contrastive Loss on the embeddings. The data we used is the validation dataset of ImageNet which contains 50000 captioned images with 1000 distinct captions. We used 8 trained models listed in \cref{tab:siglip_results_5percent} which can all be downloaded from \href{https://huggingface.co/google/siglip-base-patch16-224}{Hugging Face.}

We embedded all images and labels in the validation set using the B/16 model. We used PIL to resize all images to 224x24. In Figure \ref{fig:siglip_ipseparation}, we show in red the inner products between wrong image-caption pairs and in blue between correct image-caption pairs.

As we point out, the inner product separation is nearly satisfied. There are some errors, but such are expected. For example, we discovered that the best matching image embeddings picture of the word ``African chameleon'' was ``American chameleon''. Both are species of chameleon and, hence, the images similar, such errors are to be expected in practice. For large models, the reported accuracy on ImageNet in \cite{zhai23siglip} is 84.5\%.

\subsection{Experiments with Locked Representation}\label{appendix:lockedexperiments}
In Figure \ref{fig:lockedencoder}, we performed experiments in which one modality is fixed. Namely, we first draw $\{U_i\}_{i = 1}^N$ uniformly on the sphere and then fix them. Then, we try to synchronize with 
$\{V_i\}_{i = 1}^N$ by running gradient descent on the respective loss function. Specifically, we have experiments on:
\begin{itemize}
    \item \emph{Fixed Low Temperature $t = 200$ and bias $b = 0.$} We fix $t = 200, b = 0$ and run Adam on $\{V_i\}_{i = 1}^N$ for the loss  $\losssiglip(\{U_i\}_{i = 1}^N, \{V_i\}_{i = 1}^N; t,b)$ and initial learning rate $0.01.$
    \item \emph{Fixed High Temperature $t = 10$ and bias $b = 0.$} We fix $t = 10, b = 0$ and run Adam
    on $\{V_i\}_{i = 1}^N$ for the loss
    $\losssiglip(\{U_i\}_{i = 1}^N, \{V_i\}_{i = 1}^N);t,b$ and initial learning rate $0.01.$
    \item \emph{Trainable Temperature and Bias.} We initialize at $t = 10 = e^{t'}, b = 0$ and run Adam with on $\{V_i\}_{i = 1}^N, t', b$ for the loss $\losssiglip(\{U_i\}_{i = 1}^N, \{V_i\}_{i = 1}^N;e^{t'}, b)$ and initial learning rate $0.01.$
    We note that all of our trainable experiments are with the parametrization $t = e^{t'}$ which ensures positive temperature as in \cite{zhai23siglip}.
    \item \emph{Trainable Temperature and Relative Bias.} We initialize at $t = 10 = e^{t'}, b = 0$ and run Adam on $\{V_i\}_{i = 1}^N, t', \relativebias$ for the loss $\losssigliprb(\{U_i\}_{i = 1}^N, \{V_i\}_{i = 1}^N;e^{t'}, \relativebias)$ and initial learning rate $0.01.$
    We note that all of our trainable experiments are with the parametrization $t = e^{t'}$ which ensures positive temperature as in \cite{zhai23siglip}.
\end{itemize}
The specific experiment in \cref{fig:lockedrepresenattion} is for $d = 10,N = 100.$ We also note that we did one more comparison, which  is not reported in the main paper -- with an explicit adapter from Figure \ref{fig:lockedencoder}. Namely:
\begin{itemize}
    \item \emph{Trainable Temperature and Relative Bias with Explicit Adapter.} We initialize at $t = 10 = e^{t'}, b = 0, \delta = \frac{e^x}{1+e^x}$ with $x= 1/2$ and run Adam on $\{V_i\}_{i = 1}^N, t', \relativebias,x$ for the loss $\losssigliprb(\{\adapter^\delta_{\mathsf{locked}}(U_i)\}_{i = 1}^N, \{\adapter^\delta_{\mathsf{trainable}}(V_i)\}_{i = 1}^N;e^{t'}, \relativebias)$ and initial learning rate $0.01.$ Since the adapter is an invertible transformation on the representations, we reported the inner products both with the adapter and without it (that is, we invert by removing the last coordinate and dividing by $\delta.$)
\end{itemize}

\begin{figure}[htbp]
  \centering
  \begin{subfigure}[b]{0.48\textwidth}
    \centering
    \includegraphics[width=\textwidth]{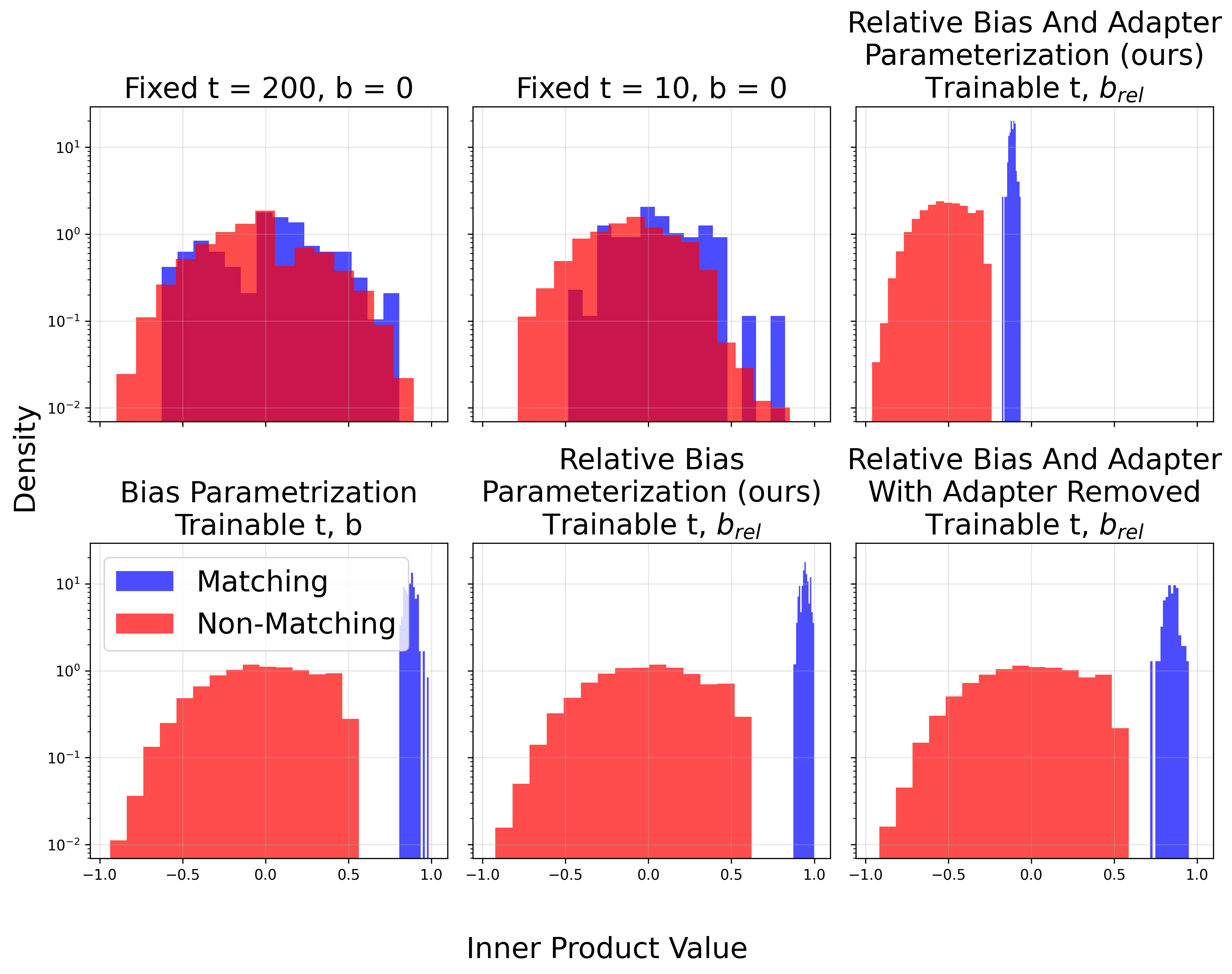}
    % \caption{Log‐density histograms of inner‐product scores for matching vs.\ non‐matching pairs under six parameterizations.}
    \label{fig:densities}
  \end{subfigure}
  \hfill
  \begin{subfigure}[b]{0.48\textwidth}
    \centering
    \includegraphics[width=\textwidth]{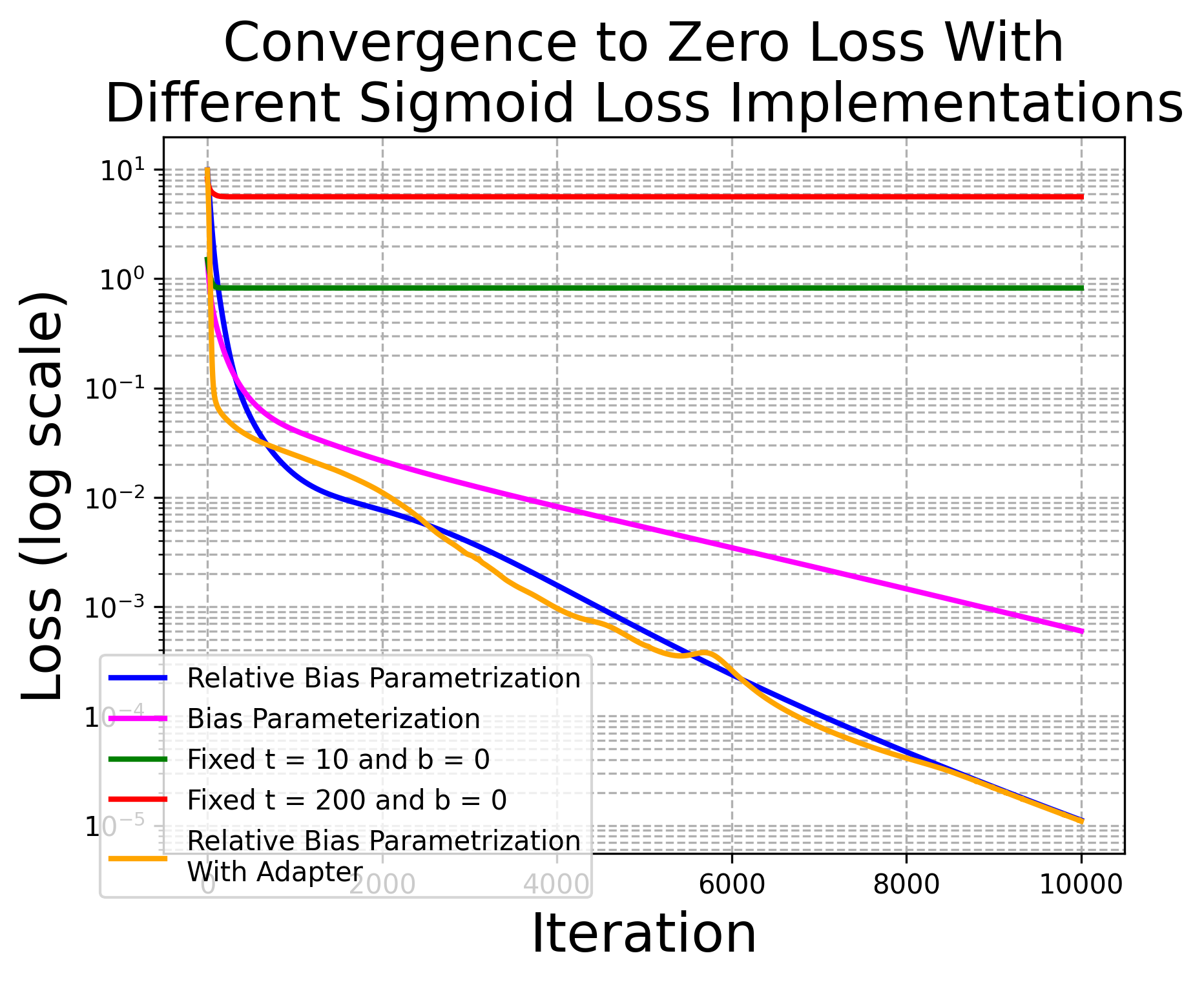}
    % \caption{Sigmoid‐loss trajectories (log scale) over 10 000 iterations under the same settings.}
    \label{fig:loss}
  \end{subfigure}

  \caption{Inner‐product separation and loss convergence under six sigmoid‐loss parameterizations. \textit{Left:} Log‐density histograms of inner‐product scores for matching (blue) versus non‐matching (red) pairs, evaluated under fixed inverse temperature $t=200,b=0$, fixed $t=10,b=0$, trainable bias $b$, our relative‐bias parameterization (trainable $b_{\mathrm{rel}}$), and the same two schemes with the adapter removed; only the trainable‐bias models show clear separation. \textit{Right:} Sigmoid‐loss trajectories (log scale) over 10{,}000 iterations for the same six settings; only those variants that learn both bias and inverse temperature reach zero loss, and our relative‐bias parameterization (with and without adapter) converges most rapidly.}

  \label{fig:all_results}
\end{figure}

We can overall see that the performance of $\losssigliprb$ algorithm with an adapter and without is rather comparable and the inner product separations are similar. One difference to note is that the training with adapter seems less stable.  Thus, we believe that in practice not using the adapter might be the better approach. 

\subsection{Experiments with Multiple Modalities}
In Figure \ref{fig:multiplemodalities}, we performed experiments with $k = 4$ modalities. Namely, we synchronize 
$\{(U^{(1)}_i, U^{(2)}_i, U^{(3)}_i, U^{(4)}_i)\}_{i = 1}^N$ by running gradient descent on the sums of all pairwise loss functions between the 4 modalities. Specifically, we have experiments on:
\begin{itemize}
    \item \emph{Fixed Low Temperature $t = 200$ and bias $b = 0.$} We fix $t = 200, b = 0$ and run Adam on $\{V_i\}_{i = 1}^N$ for the loss  $\losssiglip(\{U_i\}_{i = 1}^N, \{V_i\}_{i = 1}^N; t,b)$ and initial learning rate $0.01.$
    \item \emph{Fixed High Temperature $t = 10$ and bias $b = 0.$} We fix $t = 10, b = 0$ and run Adam
    on $\{V_i\}_{i = 1}^N$ for the loss
    $\losssiglip(\{U_i\}_{i = 1}^N, \{V_i\}_{i = 1}^N);t,b$ and initial learning rate $0.01.$
    \item \emph{Trainable Temperature and Bias.} We initialize at $t = 10 = e^{t'}, b = 0$ and run Adam with on $\{V_i\}_{i = 1}^N, t', b$ for the loss $\losssiglip(\{U_i\}_{i = 1}^N, \{V_i\}_{i = 1}^N;e^{t'}, b)$ and initial learning rate $0.01.$
    We note that all of our trainable experiments are with the parametrization $t = e^{t'}$ which ensures positive temperature as in \cite{zhai23siglip}.
    \item \emph{Trainable Temperature and Relative Bias.} We initialize at $t = 10 = e^{t'}, b = 0$ and run Adam on $\{V_i\}_{i = 1}^N, t', \relativebias$ for the loss $\losssigliprb(\{U_i\}_{i = 1}^N, \{V_i\}_{i = 1}^N;e^{t'}, \relativebias)$ and initial learning rate $0.01.$
    We note that all of our trainable experiments are with the parametrization $t = e^{t'}$ which ensures positive temperature as in \cite{zhai23siglip}.
\end{itemize}
The specific experiment in \cref{fig:multiplemodalities} is for $d = 10,N = 100.$

We ran additional experiments to investigate how increasing the number of modalities $k$ affects the final separation margin. With trainable temperature and relative bias, we observe that the margin generally increases as we synchronize more modalities, as summarized in \cref{tab:modalities_margin}. This suggests that training with more modalities may lead to more robust representations, as a larger margin implies better separation between matching and non-matching pairs.

\begin{table}[h]
\centering
\caption{Final margin as a function of the number of modalities being synchronized. The experiment was run with $N=100$ and $d=10$.}
\label{tab:modalities_margin}
\begin{tabular}{@{}cc@{}}
\toprule
\textbf{Number of Modalities} & \textbf{Final Margin} \\ \midrule
2  & 0.471241 \\
4  & 0.427528 \\
6  & 0.472571 \\
8  & 0.595576 \\
14 & 0.610853 \\
20 & 0.611314 \\ \bottomrule
\end{tabular}
\end{table}
\subsection{Bias Parameterization Leads to Zero Relative Bias}
\label{appendix:biasparametrization}
Finally, we do experiments to show that training with $\losssiglip$ leads to near zero relative bias, as in \cite{siglip2demo}. We compare with $\losssigliprb.$ Concretely, we run experiments with $N = 100$ points 
$\{(U_i, V_i)\}_{i = 1}^N$ initialized at random and run Adam on $\losssiglip(\{(U_i, V_i)\}_{i = 1}^N; t, b),$ respectively 
$\losssigliprb(\{(U_i, V_i)\}_{i = 1}^N; t, \relativebias),$ for 10000 epochs starting at $t = 10$ and varying biases. 

We compare the evolution of relative biases, inverse temperature, loss function, and margins of the final configuration. 
\begin{figure}[!htb]
    \centering
    \includegraphics[width=0.9\linewidth]{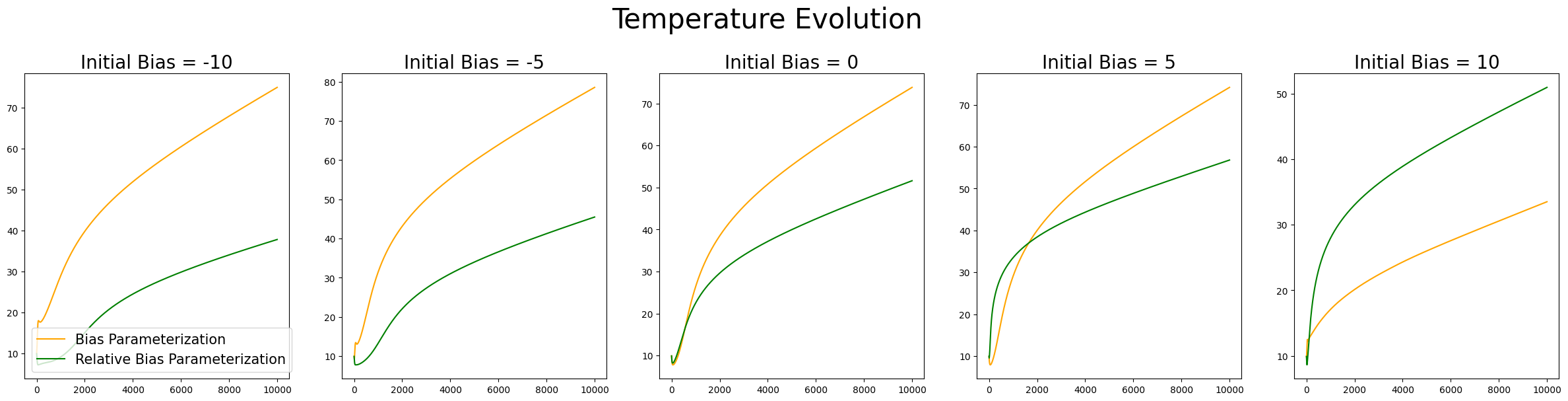}
    \caption{Evolution of the inverse temperature parameter during the training process.}
    \label{fig:temperature_in_appendix_d}
\end{figure}
\begin{figure}[!htb]
    \centering
    \includegraphics[width=0.9\linewidth]{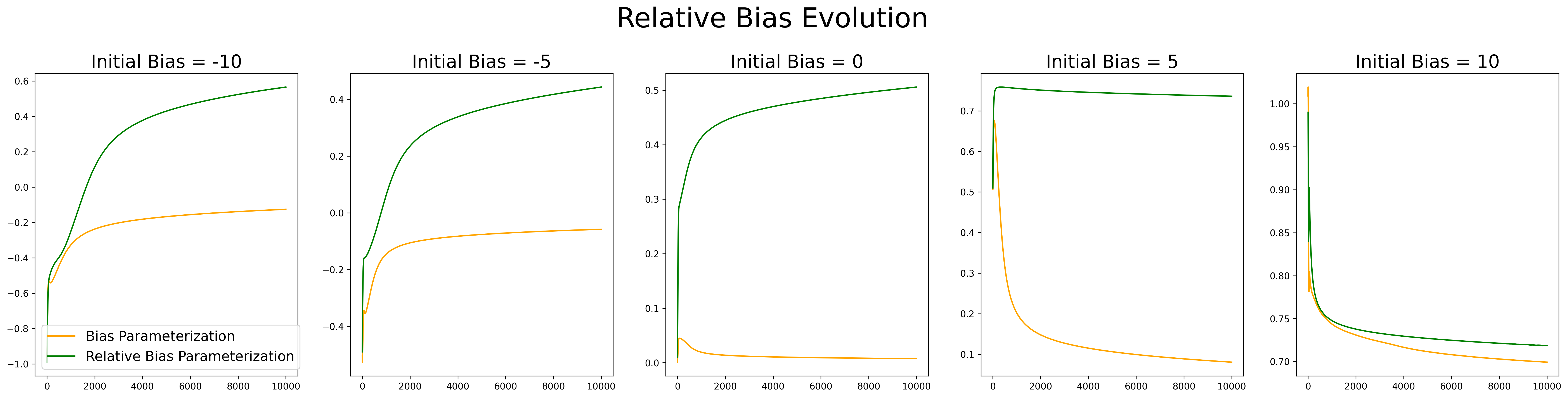}
    \caption{\small{Relative bias is in general smaller when training with the $\losssiglip$ parameterization. In general, it converges to zero and is significantly smaller than the relative bias of the $\losssigliprb.$}}
    \label{fig:relativebias_in_appendix_c}
\end{figure}

\begin{figure}[!htb]
    \centering
    \includegraphics[width=0.9\linewidth]{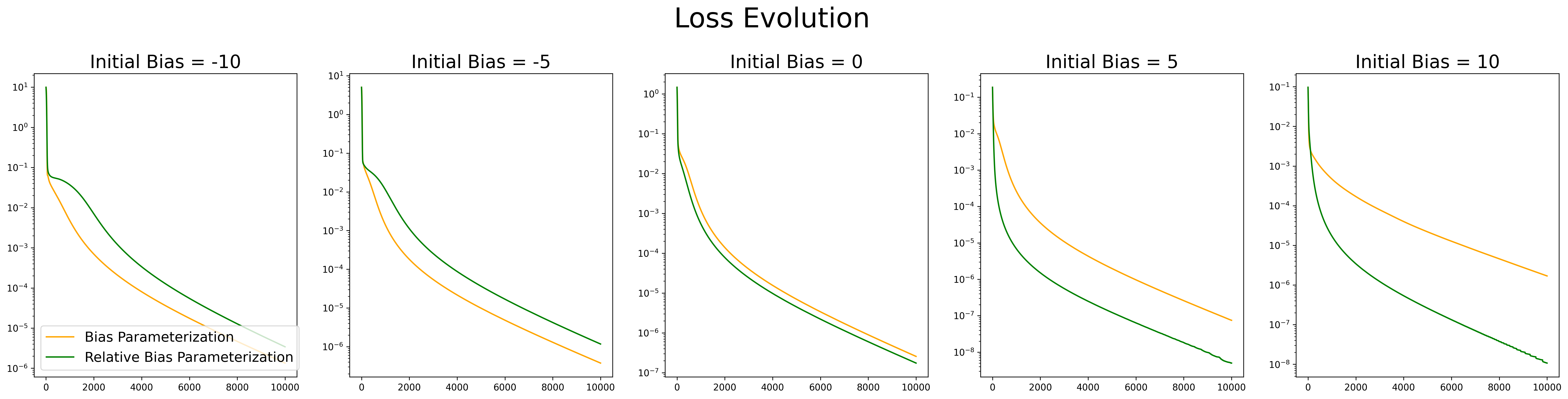}
    \caption{\small{In general, the loss converges faster to zero when trained with the $\losssigliprb$ parameterization than when trained with $\losssiglip.$}}
    \label{fig:relativebias_losses_appendix_c}
\end{figure}

\begin{figure}[!htb]
    \centering
    \includegraphics[width=.5\textwidth]{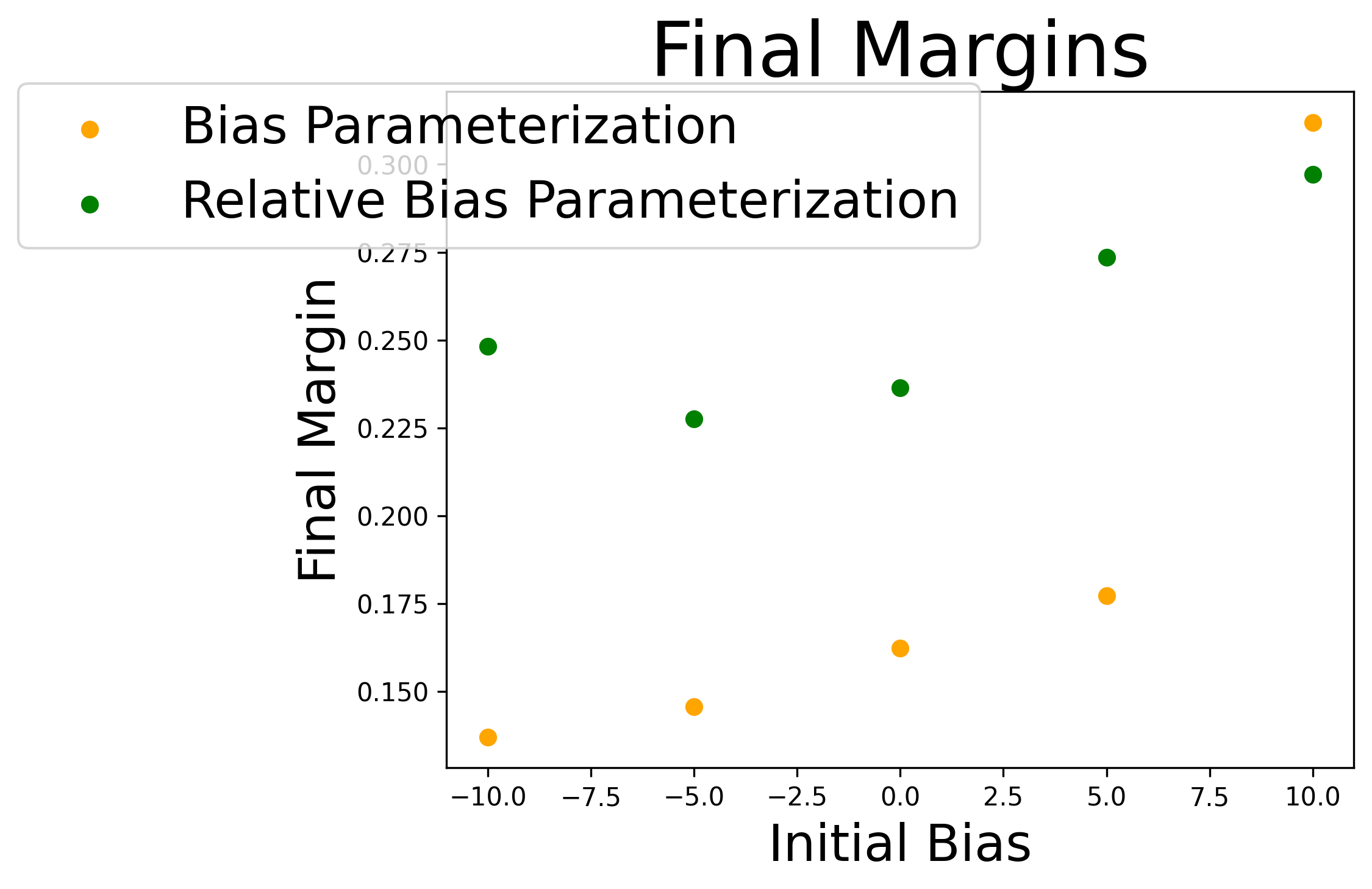}
    \caption{\small{In general, the margin is much larger for representations trained with the $\losssigliprb$ parameterization. As we know from \ref{cor:perfectretrieval}, this means that they are more robust on retrieval tasks.}}
    \label{fig:relativebias}
\end{figure}

Finally, we also compare the margins. The reason is that as we know from \cref{thm:upperboundsmargin}, there is an important relationship between relative bias and margin. The fact that embeddings trained with $\losssigliprb$ have a larger relative bias also impacts the margin. 

\subsection{Initializing Fixed Relative Bias}
We verify this with an experiment where we initialize representations uniformly at random, fix the relative bias $\relativebias$, and train the representations and inverse temperature $t$ using Adam. As shown in \cref{tab:fixed_relative_bias}, choosing $\relativebias \approx 0.7$ yields the largest final margin, while other choices result in smaller margins. This confirms that the relative bias parameter can effectively steer the optimization towards configurations with desirable properties.

\begin{table}[h]
\centering
\caption{Final margin and loss for different fixed values of relative bias $\relativebias$. Training is performed on the representations and inverse temperature $t$. The largest margin is achieved for $\relativebias \approx 0.7$.}
\label{tab:fixed_relative_bias}
\begin{tabular}{@{}cccc@{}}
\toprule
\textbf{Fixed Relative Bias} & \textbf{Final Temperature} & \textbf{Achieved Margin} & \textbf{Final Loss} \\ \midrule
-1.00 & 6.961601 & -0.000001 & 0.693150 \\
-0.90 & 56.188858 & -0.000000 & 0.009245 \\
-0.80 & 23.300198 & -0.000000 & 0.014105 \\
-0.70 & 162.664841 & 0.092437 & 0.000005 \\
-0.60 & 125.656731 & 0.122326 & 0.000003 \\
-0.50 & 104.095329 & 0.152893 & 0.000003 \\
-0.40 & 90.788620 & 0.182992 & 0.000002 \\
-0.30 & 81.079422 & 0.213618 & 0.000001 \\
-0.20 & 75.061546 & 0.242438 & 0.000001 \\
-0.10 & 71.254242 & 0.273600 & 0.000000 \\
0.00 & 69.018265 & 0.301340 & 0.000000 \\
0.10 & 69.534515 & 0.329022 & 0.000000 \\
0.20 & 67.996437 & 0.353406 & 0.000000 \\
0.30 & 61.796310 & 0.390087 & 0.000000 \\
0.40 & 55.452553 & 0.430921 & 0.000000 \\
0.50 & 48.406261 & 0.466707 & 0.000000 \\
0.60 & 44.391388 & 0.498564 & 0.000000 \\
0.70 & 42.265457 & 0.527834 & 0.000000 \\
0.80 & 37.361767 & 0.539749 & 0.000001 \\
0.90 & 33.167175 & 0.483351 & 0.000036 \\
1.00 & 23.817665 & 0.513416 & 0.000693 \\ \bottomrule
\end{tabular}
\end{table}
\subsection{Initializing Learnable Temperature and Relative Bias.} We investigated the effect of initial temperature $t$ and relative bias $\relativebias$ on the final margin. We ran a hyperparameter search and found that the final margin is best for a small initial temperature ($t \le 3$) or an intermediate temperature ($t \approx 10$) with a relatively large initial relative bias ($\relativebias \approx 0.6$). The results, summarized in \cref{tab:initial_params}, show that while the optimization is robust to a range of initializations, a poor choice (e.g., high initial $t$ and low $\relativebias$) can lead to suboptimal final representations with a small or even negative margin.

\begin{table}[h!]
\centering
\caption{The final margin achieved for different initializations of temperature (Temp) and relative bias ($\relativebias$). The best results (bolded) are obtained with low-to-intermediate temperature and high relative bias.}
\label{tab:initial_params}
\small % Use a smaller font size for the table
\setlength{\tabcolsep}{4pt} % Reduce the space between columns
\begin{tabular}{@{}lccccccccccc@{}} % Corrected to 12 columns (1 'l' + 11 'c')
\toprule
\textbf{Temp} & \textbf{-1.0} & \textbf{-0.8} & \textbf{-0.6} & \textbf{-0.4} & \textbf{-0.2} & \textbf{0.0} & \textbf{0.2} & \textbf{0.4} & \textbf{0.6} & \textbf{0.8} & \textbf{1.0} \\ 
\midrule
1    & 0.567 & 0.567 & 0.566 & 0.564 & 0.566 & 0.566 & 0.565 & 0.568 & 0.570 & \textbf{0.574} & 0.573 \\
3    & 0.545 & 0.543 & 0.526 & 0.499 & 0.483 & 0.488 & 0.536 & 0.563 & 0.570 & \textbf{0.574} & 0.573 \\
10   & 0.439 & 0.425 & 0.415 & 0.406 & 0.402 & 0.410 & 0.429 & 0.524 & 0.566 & 0.547 & 0.562 \\
30   & 0.301 & 0.297 & 0.294 & 0.315 & 0.343 & -0.942 & -0.915 & -0.935 & -1.171 & -1.051 & -1.145 \\
100  & -0.774 & -0.679 & -0.483 & -0.740 & -0.878 & -0.956 & -0.978 & -1.109 & -1.186 & -1.080 & -1.449 \\ 
\bottomrule
\end{tabular}
\end{table}

\section{Connection to Linear Representation Hypothesis Across Modalities}
\label{appendix:LRH}
It has been observed by many authors that modern dense embedding spaces acquire correspondence between linear-algebraic operations and real-world concepts. This has been immortalized as ``King - Man + Woman $\approx$ Queen'' in word2vec \cite{mikolov2013efficientestimationwordrepresentations} 
% but is also widely observed in mechanistic interpretability of LLMs \cite{templeton2024scaling} for more.
and is also observed in modern LLMs as well
\cite{park24LRH, templeton2024scaling}.
Curiously, we find that contrastive pretraining with sigmoid loss also leads to a special case of LRH: there emerges a direction $\bar x$ such that adding it to an image embedding (almost) recovers the embedding of a matching text caption. Indeed, looking at the optimal embeddings in Fig.~\ref{fig:basicfigure} we can see that $U_i-V_i$ does not depend on $i$, which we take as a manifestation of LRH in this context (the concept being ``shift text to image'' or more generally one modality to another).  Furthermore, both our upper and lower bounds on the cardinality of the embeddings in \cref{sec:cardinalitybounds} require Cross-Modality-LRH satisfying configurations to be tight.

In the proof of \cref{thm:upperboundconst} in \cref{sec:upperboundconstellationsize} we defined the following quantity characterizing an arbitrary constellation
$$ \xi =\frac{1}{N}\Bigl(\sum_{i}\|x_i\|^2\Bigr) - \|\bar x\|^2\ge0\,,
$$
where $x_i=U_i-V_i$ and $\bar x = {1\over N} \sum_i x_i$. (In the proof $\xi$ was defined for a carefully chosen sub-constellation). We note that the upper bound in that \cref{thm:upperboundconst} could only possibly be tight if $\xi \approx 0$. In this section we will further show that $\xi$ can be used as a quantitative measure of the degree to which \textit{Linear Representation Hypothesis} is satisfied. 
%sIn this way, 

First, let us establish that when $\xi$ is small (as $\xi \ge 0$ is used in the proof of \cref{thm:upperboundconst} this means that the bounds in that proof are tight). More importantly, a small $\xi$ implies something significant about our representations: it suggests that the difference vector, $U_i - V_i$, is nearly identical for all indices $i$. Think of it this way: if you have two sets of learned representations, say $U_i$ for images and $V_i$ for their corresponding text descriptions, a small $\xi$ means that you can apply a consistent shift (a single vector) to all the $V_i$ vectors to transform them into their corresponding $U_i$ vectors. So, by shifting a text representation, you could get its corresponding image representation.

\begin{proposition}
    If $\xi = o(1)$ then $\frac{1}{N} \sum_{i=1}^N \|x_i - \bar{x}\|^2 = o(1)$ and all pairs of representations align in the sense that $U_i - V_i \approx U_j-V_j$ for all $i,j$. In particular, the $U$'s are obtained from the $V$'s by adding a vector $\bar x$, which thus serves as a \textit{concept shift}.
\end{proposition}
\begin{proof}
A simple algebra shows that $\xi$ has the following two equivalent expressions:
$$ \xi = \frac{1}{N} \sum_{i=1}^N \|x_i - \bar{x}\|^2  = {1\over 2N^2} \sum_{i,j=1}^{N} \|x_i - x_j\|^2\,.$$
Thus, the statement "$\xi = o(1) $" is equivalent to
$$ 0 \le \frac{1}{N} \sum_{i=1}^N \|x_i - \bar{x}\|^2  \to 0\,,$$
which in turn implies that $U_i - V_i \approx \bar x$ simultaneously for all $i$. The argument for $U_i - V_i \approx U_j - V_j$ is similar.
%This directly proves that the average squared deviation of $x_i$ from their mean $\bar{x}$ goes to zero as $N \to \infty$. In other words, $\frac{1}{N} \sum_{i=1}^N \|x_i - \bar{x}\|^2 = o(1)$. This implies that the vectors $x_i$ tend to cluster around their mean $\bar{x}$. If they cluster around their mean, they are all close to each other.
\end{proof}

\begin{corollary}
If $\xi = 0$, then $U_i - V_i$ are all identical for all indices $i$.
\end{corollary}
\iffalse
\textbf{Proof:}
If $\xi = 0$, then from the relationship derived above:
$$ 0 = \frac{1}{N} \sum_{i=1}^N \|x_i - \bar{x}\|^2 $$
Since $\|x_i - \bar{x}\|^2 \ge 0$, for their sum to be zero, each term in the sum must be zero.
Therefore, for all $i=1, \dots, N$:
$$ \|x_i - \bar{x}\|^2 = 0 $$
This implies that $x_i - \bar{x} = 0$, or $x_i = \bar{x}$ for all $i$.
Since $x_i = U_i - V_i$, this means that $U_i - V_i = \bar{x}$ for all $i$. In other words, the difference vector $U_i - V_i$ is constant across all indices $i$.
$\hfill \blacksquare$
\fi

Indeed, in \cref{fig:xievolution} when training directly the representations, we can observe the $\xi$ value converging to 0 for a range of different dimensions.

\begin{figure}[!htb]
    \centering
    \includegraphics[width=\linewidth]{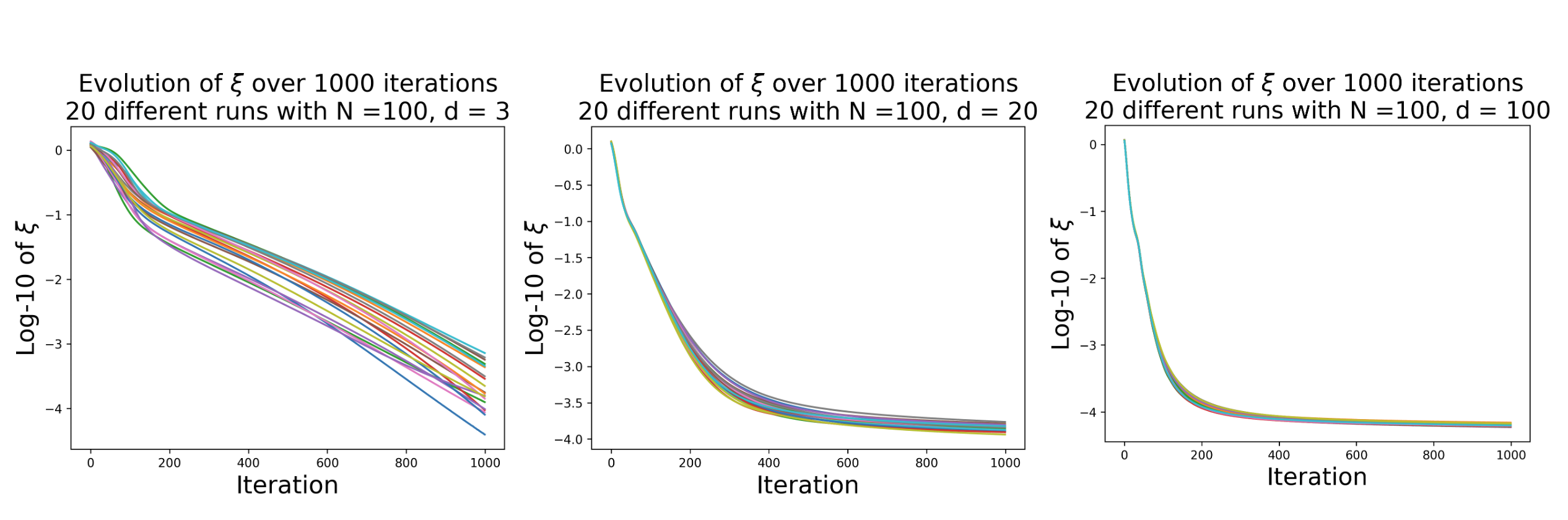}
    \caption{Convergence of the value of $\xi$ to zero during an experiment with $d = 10$ and $100$ $U_i, V_i$ pairs trained with SigLIP.}
    \label{fig:xievolution}
\end{figure}

However, the value of $\xi$ for real models is far from 0 on the ImageNet validation dataset. Our intuition for this is that the dimension used $d\approx 1000$ is far from optimal, hence $\xi = 0$ is not required. It is an interesting open direction whether we can train models in lower dimension utilizing the fact that $\xi \longrightarrow 0$ in that case. For example, one can explicitly add $\xi$ in the loss function.

\begin{table}[!htb]
\centering
\label{tab:siglip_results_norms}
\small % Use a smaller font size for the table
\setlength{\tabcolsep}{4pt} % Reduce the space between columns
\begin{tabular}{@{}lcccc@{}} % 1 'l' + 4 'c' = 5 columns
\toprule
\textbf{Model} & \boldmath$\xi$ & \textbf{Mean of Norms} & \textbf{Norm of Mean} & \textbf{Random Mean of Norms} \\
\midrule
siglip-so400m-patch14-384  & 0.6086 & 1.7249 & 1.1162 & 2.0029 \\
siglip-base-patch16-224   & 0.5880 & 1.8100 & 1.2221 & 2.0609 \\
siglip-base-patch16-384   & 0.5908 & 1.8068 & 1.2160 & 2.0631 \\
siglip-large-patch16-256  & 0.5535 & 1.7955 & 1.2420 & 2.0711 \\
siglip-so400m-patch14-224 \color{black} & 0.6207 & 1.7270 & 1.1063 & 2.0038 \\
siglip-base-patch16-256   & 0.5767 & 1.7991 & 1.2225 & 2.0588 \\
siglip-base-patch16-512   & 0.5908 & 1.8059 & 1.2151 & 2.0644 \\
siglip-large-patch16-384  & 0.5744 & 1.8084 & 1.2340 & 2.0762 \\
\bottomrule
\end{tabular}
\caption{$\xi$ for different SigLIP models in ImageNet validation. We plot respectively $\xi,$
$\frac{1}{N}\Bigl(\sum_{i}\|x_i\|^2\Bigr)$ as mean of norms, $\|\bar x\|^2$ as norm of mean, 
$\frac{1}{N}\Bigl(\sum_{i}\|U_i - V_{\pi(i)}\|^2\Bigr)$ as random mean of norms where $\pi$ is a uniformly random permutation. We can see that the mean of norms is closer to the random mean of norms (i.e., random pairing of text-images, not corresponding to the ground truth) rather than the norm of means, which would imply $\xi\longrightarrow 0.$ 
}
\end{table}
\end{document}